\documentclass{article}

\usepackage{microtype}
\usepackage{graphicx}
\usepackage{subfigure}
\usepackage{booktabs} 

\usepackage{hyperref}

\usepackage[accepted]{icml2024}

\usepackage{amsmath}
\usepackage{amssymb}
\usepackage{mathtools}
\usepackage{amsthm}

\usepackage[capitalize,noabbrev]{cleveref}

\usepackage{url}            
\usepackage{amsfonts}       
\usepackage{adjustbox}
\usepackage{bm}
\usepackage{caption}
\usepackage{comment}
\usepackage{subcaption}
\usepackage{float}
\usepackage{threeparttable}
\usepackage{makecell}
\usepackage{multirow}
\usepackage{wrapfig}
\usepackage{enumitem}
\usepackage{algorithm}
\usepackage{algorithmic}
\usepackage{tabularx}
\usepackage{enumitem}
\usepackage{ulem}
\usepackage{subcaption}
\usepackage{tikz}

\newcommand{\prompt}[1]{\tiny #1}
\newcommand{\orange}[1]{\textcolor{orange}{#1}}

\theoremstyle{plain}
\newtheorem{theorem}{Theorem}[section]
\newtheorem{proposition}[theorem]{Proposition}

\theoremstyle{definition}
\newtheorem{definition}[theorem]{Definition}

\theoremstyle{remark}

\usepackage[textsize=tiny]{todonotes}

\begin{document}

\twocolumn[
\icmltitle{On Discrete Prompt Optimization for Diffusion Models}


\icmlsetsymbol{equal}{$\dagger$}

\begin{icmlauthorlist}
\icmlauthor{Ruochen Wang}{google,ucla}
\icmlauthor{Ting Liu}{deepmind}
\icmlauthor{Cho-Jui Hsieh}{google,ucla}
\icmlauthor{Boqing Gong}{google}\\
{Google Research} \quad {Google Deepmind} \quad {UCLA} \\
\url{https://github.com/ruocwang/dpo-diffusion}
\end{icmlauthorlist}

\icmlaffiliation{ucla}{University of California, Los Angeles}
\icmlaffiliation{google}{Google Research}
\icmlaffiliation{deepmind}{Google Deepmind}

\icmlcorrespondingauthor{Boqing Gong}{bgong@google.com}
\icmlcorrespondingauthor{Ruochen Wang}{ruocwang@g.ucla.edu}

\icmlkeywords{Machine Learning, ICML}

\vskip 0.3in
]



\printAffiliationsAndNotice{}  

\begin{abstract}
This paper introduces the first gradient-based framework for prompt optimization in text-to-image diffusion models. We formulate prompt engineering as a discrete optimization problem over the language space. Two major challenges arise in efficiently finding a solution to this problem: \textit{(1) Enormous Domain Space:} Setting the domain to the entire language space poses significant difficulty to the optimization process. \textit{(2) Text Gradient:} Efficiently computing the text gradient is challenging, as it requires backpropagating through the inference steps of the diffusion model and a non-differentiable embedding lookup table. Beyond the problem formulation, our main technical contributions lie in solving the above challenges. First, we design a family of dynamically generated compact subspaces comprised of only the most relevant words to user input, substantially restricting the domain space. Second, we introduce ``Shortcut Text Gradient" --- an effective replacement for the text gradient that can be obtained with constant memory and runtime. Empirical evaluation on prompts collected from diverse sources (DiffusionDB, ChatGPT, COCO) suggests that our method can discover prompts that substantially improve (prompt enhancement) or destroy (adversarial attack) the faithfulness of images generated by the text-to-image diffusion model.
\end{abstract}
\section{Introduction}
\label{sec:intro}

\begin{figure*}[t]
    \centering
    \includegraphics[width=0.95\textwidth,
    clip=false]{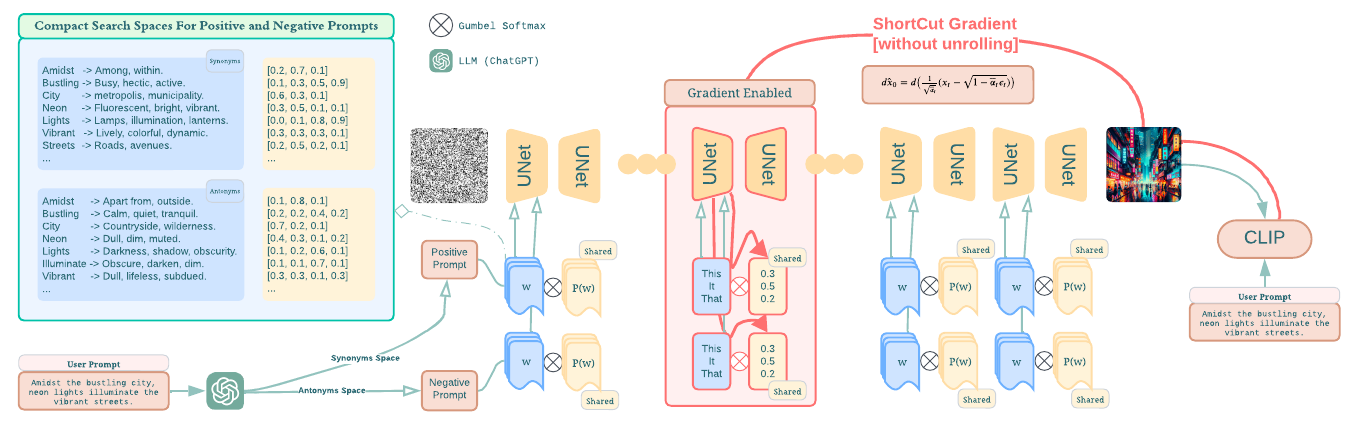}
    \caption{Computational procedure of Shortcut Text Gradient (Bottom) v.s. Full Gradient (Top) on text.}
    \label{fig:DPO-Diff-pipeline}
    \vspace{-2mm}
\end{figure*}


Large-scale text-based generative models exhibit a remarkable ability to generate novel content conditioned on user input prompts~\citep{chatgpt, llama, sd, dalle2, imagen, imagen-video, parti, muse}.
Despite being trained with huge corpora, there still exists a substantial gap between user intention and what the model interprets~\citep{ape, structdiff, sd, clip, grounded, chatgpt, dalle2}.
The misalignment is even more severe in text-to-image generative models, partially since they often rely on much smaller and less capable text encoders~\citep{clip, openclip, t5} than large language models (LLMs).
As a result, instructing a large model to produce intended content often requires laborious human efforts in crafting the prompt through trials and errors (a.k.a.\ Prompt Engineering)~\citep{lexica, diffusiondb, pe-diff, pe-diff2, ape, promptist}.
To automate this process for language generation, several recent attempts have shown tremendous potential in utilizing LLMs to enhance prompts~\citep{apo, ape, instructzero, prompt-evo, llm-opt, promptist}.
However, efforts on text-to-image generative models remain scarce and preliminary, probably due to the challenges faced by these models' relatively small text encoders in understanding subtle language cues.

\paragraph{DPO-Diff.}
This paper presents a systematic study of prompt optimization for text-to-image diffusion models.
We introduce a novel optimization framework based on the following key observations.
\textit{1) Prompt engineering for diffusion models can be formulated as a Discrete Prompt Optimization (DPO-Diff) problem over the space of natural languages.}
Moreover, the framework can be used to find prompts that either improve (prompt enhancement) or destroy (adversarial attack) the generation process, by simply reversing the sign of the objective function.
\textit{2) We show that for diffusion models with classifier-free guidance~\citep{cfg}, improving the image generation process is more effective when optimizing ``negative prompts''~\citep{np1, np2} than positive prompts.}
Beyond the problem formulation of DPO-Diff, where ``Diff'' highlights our focus on text-to-image diffusion models, the main technical contributions of this paper lie in efficient methods for solving this optimization problem, including the design of compact domain spaces and a gradient-based algorithm.

\paragraph{Compact domain spaces.}
DPO-Diff's domain space is a discrete search space at the word level to represent prompts.
While this space is generic enough to cover any sentence, it is excessively large due to the dominance of words irrelevant to the user input.
To alleviate this issue, we design a family of dynamically generated compact search spaces based on relevant word substitutions, for both positive and negative prompts.
These subspaces enable efficient search for both prompt enhancement and adversarial attack tasks.

\paragraph{Shortcut Text Gradients for DPO-Diff.}
Solving DPO-Diff with a gradient-based algorithm requires computing the text gradient, i.e., backpropagating from the generated image, through all inference steps of a diffusion model, and finally to the discrete text.
Two challenges arise in obtaining this gradient:
1) This process incurs compound memory-runtime complexity over the number of backward passes through the denoising step, making it prohibitive to run on large-scale diffusion models
(e.g., a 870M-parameter Stable Diffusion v1 requires $\sim$750G memory to run backpropagation through 50 inference steps~\citep{sd}).
2) The embedding lookup tables in text encoders are non-differentiable.
To reduce the computational cost in 1), we provide a generic replacement for the text gradient that bypasses the need to unroll the inference steps in a backward pass, allowing it to be computed with constant memory and runtime.
To backpropagate through the discrete embedding lookup table, we continuously relax the categorical word choices to a learnable smooth distribution over the vocabulary, using the Gumbel Softmax trick~\citep{gbda, gumbel, gdas}.
The gradient obtained by this method, termed \textbf{Shortcut Text Gradient}, enables us to efficiently solve DPO-Diff regardless of the number of inference steps of a diffusion model.

To evaluate our prompt optimization method for the diffusion model, we collect and filter a set of challenging prompts from diverse sources including DiffusionDB~\citep{diffusiondb}, COCO~\citep{coco}, and ChatGPT~\citep{chatgpt}.
Empirical results suggest that DPO-Diff can effectively discover prompts that improve (or destroy for adversarial attack) the faithfulness of text-to-image diffusion models, surpassing human-engineered prompts and prior baselines by a large margin.
We summarize our primary contributions as follows:
\begin{itemize}[itemsep=2pt, topsep=2pt, left=0pt, label=$\bullet$]
    \item \textbf{DPO-Diff:} A generic framework for prompt optimization as a discrete optimization problem over the space of natural languages, of arbitrary metrics.
    \item \textbf{Compact domain spaces:} A family of dynamic compact search spaces, over which a gradient-based algorithm enables efficient solution finding for the prompt optimization problem.
    \item \textbf{Shortcut Text Gradients:} The first novel computation method to enable backpropagation through the diffusion models' lengthy sampling steps with constant memory-runtime complexity, enabling gradient-based search algorithms.
    \item \textbf{Negative prompt optimization:} The first empirical result demonstrating the effectiveness of optimizing negative prompts for diffusion models.
\end{itemize}

\section{Related Work}

\paragraph{Text-to-image diffusion models.}
Diffusion models trained on a large corpus of image-text datasets significantly advanced the state of text-guided image generation~\citep{sd, dalle2, imagen, muse, parti}.
Despite the success, these models can sometimes generate images with poor quality.
While some preliminary observations suggest that negative prompts can be used to improve image quality~\citep{np1, np2}, there exists no principled way to find negative prompts.
Moreover, several studies have shown that large-scale text-to-image diffusion models face significant challenges in understanding language cues in user input during image generation;
Particularly, diffusion models often generate images with missing objects and incorrectly bounded attribute-object pairs, resulting in poor ``faithfulness'' or ``relevance''~\citep{promptist, structdiff, grounded, compdiff}.
Existing solutions to this problem include compositional generation~\citep{compdiff}, augmenting diffusion model with large language models~\citep{llm-opt}, and manipulating attention masks~\citep{structdiff}.
As a method orthogonal to them, our work reveals that negative prompt optimization can also alleviate this issue.

\paragraph{Prompt optimization for text-based generative models.}
Aligning a pretrained large language model (LLM) with human intentions is a crucial step toward unlocking the potential of large-scale text-based generative models~\citep{chatgpt, sd}.
An effective line of training-free alignment methods is prompt optimization (PO)~\citep{ape}.
PO originated from in-context learning~\citep{gpt3}, which is mainly concerned with various arrangements of task demonstrations.
It later evolves into automatic prompt engineering, where powerful language models are utilized to refine prompts for certain tasks~\citep{ape, apo, llm-opt, apo, promptist}.
While PO has been widely explored for LLMs, efforts on diffusion models remain scarce.
The most relevant prior work to ours is Promptist~\citep{promptist}, which finetunes an LLM via reinforcement learning from human feedback~\citep{chatgpt} to augment user prompts with artistic modifiers (e.g., high-resolution, 4K)~\citep{lexica}, resulting in aesthetically pleasing images.
However, the lack of paired contextual-aware data significantly limits its ability to follow the user intention (Figure \ref{fig:improve}).

\paragraph{Textual Inversion}
Optimizing texts in pretrained diffusion models has also been explored under ``Textual Inversion'' task~\citep{ti, pez, nti}.
Textual Inversion involves adapting a frozen model to generate novel visual concepts based on a set of user-provided images.
It achieves this by distilling these images into soft or hard text prompts, enabling the model to replicate the visual features of the user images.
Since the source images are provided, the training process mirrors that of typical diffusion model training.
While some Textual Inversion papers also use the term ``prompt optimization", it is distinct from the Prompt Optimization considered by Promptist~\citep{promptist} and our work.
Our objective is to enhance a model’s ability to follow text prompts.
Here, the primary input is the user prompt, and improvement is achieved by optimizing this prompt to enhance the resulting image.
Since the score function is applied to the final generated image, the optimization process necessitates backpropagation through all inference steps.
Despite using similar terminologies, these methodologies are fundamentally distinct and not interchangeable.
Table~\ref{tab:taxonomy} further summarizes the key differences in taxonomy.

\paragraph{Efficient Backpropagation through diffusion sampling steps.}
Text-to-image diffusion models generate images via a progressive denoising process, making multiple passes through the same network~\citep{ddpm}.
When a loss is applied to the output image, computing the gradient w.r.t.\ any model component (text, weight, sampler, etc.) requires backpropagating through all the sampling steps.
This process incurs compound complexity over the number of backward passes in both memory and runtime, making it infeasible to run on regular commercial devices.
Existing efforts achieve constant memory via gradient checkpointing~\citep{bp1} or solving an augmented SDE problem~\citep{bp2}, at the expense of even higher runtime.

\section{Preliminaries on diffusion model}
\label{sec:prelim}

\paragraph{Denoising diffusion probabilistic models.}
On a high level, diffusion models~\citep{ddpm} is a type of hierarchical Variational Autoencoder~\citep{hvae} that generates samples by reversing (backward) a progressive noisification process (forward).
Let $\bm{x}_0 \cdots \bm{x}_T$ be a series of intermediate samples of increasing noise levels, the forward process progressively adds Gaussian noise to the original image $\bm{x}_0$:
\begin{align}
    q(\bm{x}_t|\bm{x}_{t-1}) = \mathcal{N}(\bm{x}_{t}; \sqrt{1 - \beta_t} \bm{x}_{t-1}, \beta_t\bm{I}),
\end{align}
where $\beta$ is a scheduling variable.
Using reparameterization trick, $\bm{x}_t|^T_{t=1}$ can be computed from $\bm{x}_0$ in one step:
\begin{align}
    &\bm{x}_t = \sqrt{\bar{\alpha}_t} \bm{x}_0 + \sqrt{1 - \bar{\alpha}_t} \epsilon,\\
    &\text{where} \ \ \alpha_t = 1 - \beta_t \ \text{and} \ \bar{\alpha}_t = \prod\nolimits_{i=1}^t \alpha_i, 
    \label{eq:forward}
\end{align}
where $\epsilon$ is a standard Gaussian error.
The reverse process starts with a standard Gaussian noise, $\bm{x}_T \sim \mathcal{N}(\bm{0}, \bm{I})$, and progressively denoises it using the following joint distribution:
\begin{align*}
    p_\theta (\bm{x}_{0:T}) &= p(\bm{x}_T) \prod\nolimits_{t=1}^T p_\theta(\bm{x}_{t-1}|\bm{x}_t) \\
    &\text{where} \ \ p_\theta(\bm{x}_{t-1}|\bm{x}_t) = \mathcal{N}(\bm{x}_{t-1}; \mu_\theta(\bm{x}_t, t), \bm{\Sigma}).
\end{align*}

While the mean function $\mu_\theta(\bm{x}_t, t)$ can be parameterized by a neural network (e.g., UNet~\citep{sd, unet}) directly, prior studies found that modeling the residual error $\epsilon(\bm{x}_t, t)$ instead works better empirically~\cite{ddpm}.
The two strategies are mathematically equivalent as $\mu_\theta(\bm{x}_t, t) = \frac{1}{\sqrt{\alpha_t}}(\bm{x}_t - \frac{1-\alpha_t}{\sqrt{1-\bar\alpha_t}}\epsilon(\bm{x}_t, t))$.

\paragraph{Conditional generation and negative prompts.}
The above formulation can be easily extended to conditional generation via classifier-free guidance~\citep{cfg}, widely adopted in contemporary diffusion models.
At each sampling step, the predicted error $\Tilde\epsilon$ is obtained by subtracting the unconditional signal ($c(``")$) from the conditional signal ($c(s)$), up to a scaling factor $w$:
\begin{align}
    \Tilde\epsilon_\theta(\bm{x}_t,  c(s),\! t) = (1 + w)\epsilon_\theta(\bm{x}_t, c(s), t) - w\epsilon_\theta(x_t, c(``"), t).
\end{align}
\vspace{-1mm}
If we replace this empty string with an actual text, then it becomes a \textbf{Negative Prompt}~\citep{np1, np2}, instructing the model \uline{what to exclude from the generated image.}

\section{\textbf{DPO-Diff} Framework}
\label{sec:method}


\paragraph{Formulation}
\uline{Our main insight is that prompt engineering can be formulated as a discrete optimization problem in the language space}.
Concretely, we represent the problem domain $\mathcal{S}$ as a sequence of $M$ words $w_i$ from a predefined vocabulary $\mathcal{V}$: $\mathcal{S} = \{w_1, w_2, \dots  w_M| \forall i, \ w_i \in \mathcal{V}\}$.
This space is generic enough to cover all possible sentences of lengths less than $M$ (when the empty string is present).
Let $G(s)$ denote a text-to-image generative model, and $s_{user}$, $s$ denote the user input and optimized prompt, respectively.
The optimization problem can be written as
\begin{align}
    &\min_{s \in \mathcal{S}} \mathcal{L}(G(s), s_{user})
    \label{eq:framework}
\end{align}
where $\mathcal{L}$ can be any objective function that measures the effectiveness of the learned prompt when used to generate images.
Following previous works~\citep{promptist}, we use clip loss $\text{CLIP}(I, s_{user})$~\citep{spherical} to measure the instruction-following ability of the diffusion model.

\paragraph{Application}
\uline{DPO-Diff framework is versatile for handling not only prompt enhancement but also adversarial attack tasks.}
Figure~\ref{fig:DPO-Diff-pipeline} illustrates the taxonomy of those two applications.
Adversarial attacks for text-to-image generative models can be defined as follows:
\begin{definition}
\label{def:adv}
Given a user input $s_{user}$, the attacker aims at slightly perturbing $s_{user}$ to disrupt the prompt-following ability of image generation, i.e., the resulting generated image is no longer describable by $s_{user}$.
\end{definition}
\vspace{-1.5mm}
To modify \eqref{eq:framework} into the adversarial attack, we can simply add a negative sign to the objective function ($\mathcal{L}$), and restrict the distance between an adversarial prompt ($s$) and user input ($s_{user}$).
Mathematically, this can be written as the following:
\begin{align}
    &\min_{s \in \mathcal{S}}\; \textcolor{red}{-\mathcal{L}}(G(s), s_{user}) \quad \text{s.t.} \ \textcolor{red}{d(s, s_{user}) \leq \lambda}, 
    \label{eq:attack}
\end{align}
where $d(s, s_{user})$ is a distance measure that forces the perturbed prompt ($s$) to be semantically similar to the user input ($s_{user}$).


\section{Compact search spaces for efficient prompt discovery}
\label{sec:method.space}
While the entire language space facilitates maximal generality, it is also unnecessarily inefficient as it is popularized with words irrelevant to the task.
We propose a family of compact search spaces that dynamically extracts a subset of task-relevant words to the user input.

\subsection{Application 1: Discovering adversarial prompts for model diagnosis}
\label{sec:method.space.attack}
\paragraph{Synonym Space for adversarial attack.}
In light of the constraint on semantic similarity in \eqref{eq:attack}, we build a search space for the adversarial prompts by substituting each word in the user input $s_{user}$ with its synonyms~\citep{synonym1}, preserving the meaning of the original sentence.
The synonyms can be found by either dictionary lookup or querying ChatGPT (Appendix~\ref{app:impl.space}).

\subsection{Application 2: Discovering enhanced prompts for image generation}
\label{sec:method.space.improve}
While the Synonym Space is suitable for attacking diffusion models, we found that it performs poorly on finding improved prompts.
This is in contradiction to LLMs where rephrasing user prompts can often lead to substantial gains~\citep{ape}.
One plausible reason is that contemporary diffusion models often rely on small-scale text encoders~\citep{clip, openclip, t5} that are much weaker than LLMs with many known limitations in understanding subtle language cues~\citep{structdiff, compdiff, llm-opt}.

\vspace{-5mm}
\paragraph{Antonym Space for negative prompt optimization.}
Inspired by these observations, \uline{we propose a novel solution to optimize for negative prompts instead} --- a unique concept that rises from classifier-free guidance~\citep{cfg} used in diffusion models (Section~\ref{sec:prelim}).
Recall that negative prompts instruct the diffusion model to remove contents in generated images, opposite to the positive prompt;
Intuitively, the model's output image can safely exclude the content with the opposite meaning to the words in the user input, thereby amplifying the concepts presented in the positive prompt.
We thereby build the space of negative prompts from the antonyms of each word in the user prompt.
The antonyms of words can also be obtained either via dictionary lookup or querying ChatGPT.
However unlike synonyms space, we concatenate the antonyms directly in comma separated format, mirroring the practical usage of negative prompts.
\uline{To the best of our knowledge, this is the first exploratory work on {automated negative prompt optimization}.}

\section{A Gradient-based solver for DPO-Diff}
\label{sec:method.algo}
Due to the query efficiency of white-box algorithms leveraging gradient information, we also explore a gradient-based method to solve \eqref{eq:framework} and \eqref{eq:attack}.
However, obtaining the text gradient is non-trivial due to two major challenges.
1) Backpropagating through the sampling steps of the diffusion inference process incurs high complexity w.r.t.\ memory and runtime, making it prohibitively expensive to obtain gradients~\citep{bp1, bp2}.
For samplers with 50 inference steps (e.g., DDIM~\citep{ddim}), it raises the runtime and memory cost by \textbf{50 times} compared to a single diffusion training step.
2) To further compute the gradient on text, the backpropagation needs to pass through a non-differentiable embedding lookup table.
To alleviate these issues, we propose \textbf{Shortcut Text Gradient}, an efficient replacement for text gradient that can be obtained with \textbf{constant memory and runtime}.
Our solution to (1) and (2) are discussed in ~\Cref{sec:method.algo.backprop} and ~\Cref{sec:method.algo.gs} respectively.
Moreover, ~\Cref{sec:method.algo.es} discusses how to sample from the learned text distribution via evolutionary search.

\subsection{Shortcut Text Gradient}

\subsubsection{Backpropagating through diffusion sampling steps}
\label{sec:method.algo.backprop}
To efficiently backpropagate the loss from the final image to intermediate feature \uline{at an arbitrary step}, our key idea is to trim the computation graph down to only a few steps \uline{from both ends}, resulting in a constant number of backward passes (Figure~\ref{fig:DPO-Diff-pipeline}.
To achieve this, three operations are required through the image generation process:

\textbf{\textit{(1) Sampling without gradient from step $T$ (noise) to $t$}.}
We disable gradients up to step $t$, thereby eliminating the need for backpropagation from $T$ to $t$.

\textbf{\textit{(2) Enable gradient from $t$ to $t - K$}.}
The backward computation graph is enabled for the $K$ step starting at $t$.

\textbf{\textit{(3) Estimating $\bm{x}_0$ directly from $\bm{x}_{t-K}$}.}
To bypass the final $t - K$ steps of UNet, a naive solution is to directly decode and feed the noisy image $\bm{x}_{t-K}$ to the loss function.
However, due to distribution shifts, these intermediate images often cannot be properly interpreted by downstream modules such as VAE decoder~\cite{sd} and CLIP~\citep{apm}.
Instead, we propose to use the following closed-form estimation of the final image $\hat{\bm{x}}_0$~\cite{ddim} to bridge the gap:
\begin{align*}
    \bm{\hat x}_0 = \frac{1}{\sqrt{\Bar\alpha_{t-K}}} (\bm{x}_{t-K} - \sqrt{1 - \Bar\alpha_{t-K}}\bm{\hat\epsilon}_\theta(\bm{x}_{t-K}, t-K))
\end{align*}
This way, the Jacobian of $\hat{\bm{x}}_0$ w.r.t.\ $\bm{x}_{t-K}$ can be computed analytically, with complexity independent of $t$.
Note that the above estimation of $x_0$ is not a trick --- it directly comes from a mathematically equivalent interpretation of the diffusion model, where each inference step can be viewed as computing $\hat{\bm{x}}_0$ and plugging it into $q(\bm{x}_{t-K}|\bm{x}_{t}, \hat{\bm{x}}_0)$ to obtain the transitional probability (See Appendix~\ref{app:theory} for the derivation).

\textbf{Remark 1: Complexity Analysis}
With Shortcut Text Gradient, the computational cost of backpropagating through the inference process can be reduced to $K$-times backward passes of UNet. When we set $t = T$ and $K = T$, it becomes the full-text gradient; When $K = 1$, the computation costs reduce to a single backward pass.
\textbf{Remark 2: Connection to ReFL~\citep{refl}.}
ReFL is a post-hoc alignment method for finetuning diffusion models.
It also adopts the estimation of $x_0$ when optimizing diffusion model against a scorer, which is mathematically equivalent to the case when $K = 1$.


\subsubsection{Backpropagating through embeddings lookup table}
\label{sec:method.algo.gs}
In diffusion models, a tokenizer transforms text input into indices, which will be used to query a lookup table for corresponding word embeddings.
To allow further propagating gradients through this non-differentiable indexing operation, we relax the categorical choice of words into a continuous probability of words and learn a distribution over them.
We parameterize the distribution using Gumbel Softmax~\citep{gumbel} with uniform temperature ($\eta = 1$):
\begin{align}
    \Tilde{e} = \sum_{i=1}^{|\mathcal{V}|} e_i * \frac{\exp\left((\log\alpha_i + g_i)/\eta\right)}{\sum_{i=1}^{|\mathcal{V}|}\exp\left((\log\alpha_i + g_i)/\eta\right)}
\end{align}
where $\alpha$ (a $|\mathcal{V}|$-dimensional vector) denotes the learnable parameter, $g$ denotes the Gumbel random variable, $e_i$ is the embedding of word $i$, and $\Tilde{e}$ is the output mixed embedding.

\subsection{Efficient sampling with Evolutionary Search}
\label{sec:method.algo.es}
To efficiently sample candidate prompts from the learned Gumbel ``distribution", we adopt evolutionary search, known for its sample efficiency~\citep{ea, fbnet}.
Our adaptation of the evolutionary algorithm to the prompt optimization task involves three key steps:
\textbf{(1) Genotype Definition:} We define the genotype of each candidate prompt as the list of searched words from the compact search space, where modifications to the genotype correspond to edits the word choices in the prompt.
\textbf{(2) Population Initialization:} We initialize the algorithm's population with samples drawn from the learned Gumbel distribution to bias the starting candidates towards regions of high potential.
\textbf{(3) Evolutionary Operations:} We execute a standard evolutionary search, including several rounds of crossover and mutation~\citep{ea}, culminating in the selection of the top candidate as the optimized prompt.
Details of the complete \textbf{DPO-Diff} algorithm, including specific hyperparameters, are available in Algorithm~\ref{algo:dpo} of ~\Cref{app:algorithm} and discussed further in ~\Cref{app:impl.hyperparam}.


\paragraph{Remark: Extending DPO-Diff to Blackbox Settings.}
In cases where the model is only accessible through forward API, our Evolutionary Search (ES) module can be used as a stand-alone black-box optimizer, thereby expanding the applicability of our framework.
As further ablated in Section \ref{sec:ablate.algo}, ES archives descent results with enough queries.



\section{Experiments}
\label{sec:exp}

\begin{figure}
    \centering
 
    \includegraphics[width=0.8\linewidth]{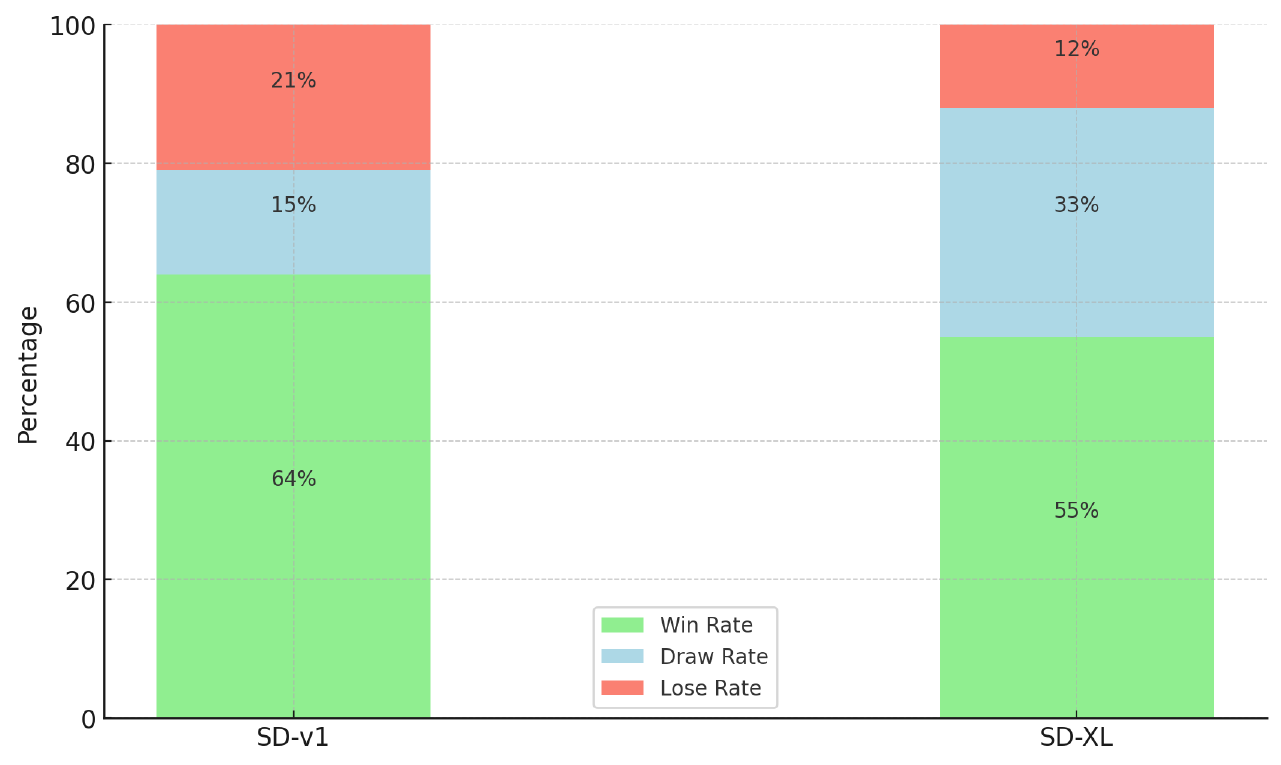}
    
    \caption{Win Rate of DPO-Diff versus Promptist on prompt improvement task with \textbf{Human Evaluation}. DPO-Diff surpasses or matches the performance of Promptist 79\% of times on SD-v1 and 88\% of times on SD-XL.}
    \label{fig:win_rate}
    
\end{figure}

\subsection{Experimental Setup}
\label{sec:setup}

    \paragraph{Dataset preparation.}
    To encourage semantic diversity, we collect a prompt dataset from three sources: DiffusionDB~\citep{diffusiondb}, ChatGPT generated prompts~\citep{chatgpt}, and COCO~\citep{coco}.
    For each source, we filter 100 \textbf{``hard prompts"} with a clip loss higher (lower for adversarial attack) than a threshold, amounting to \textbf{600 prompts} in total for two tasks.
    Due to space limit, we include preparation details in Appendix~\ref{app:exp_setting.dataset}.

    \paragraph{Evaluation Metrics.}
    All methods are evaluated quantitatively using the clip loss~\citep{vqgan-clip} and Human Preference Score v2 (HPSv2).
    HPSv2 is a CLIP-based model trained to predict human preferences on images generated from text.
    For base models, we adopt \textit{Stable Diffusion v1-4}.
    Each prompt is evaluated under two random seeds (shared across different methods).
    \textbf{Besides automatic evaluation metrics, we also conduct human evaluations on the generated images}, following the protocol specified in Appendix~\ref{app:exp_setting.evaluation}.
    
    \paragraph{Optimization Parameters.}
    We use the Spherical CLIP Loss~\citep{spherical} as the objective function, {which ranges between 0.75 and 0.85} for most inputs.
    The $K$ for the Shortcut Text Gradient is set to 1, as it produces effective supervision signals with minimal cost.
    To generate the search spaces, we prompt ChatGPT (\texttt{gpt-4-1106-preview}) for at most 5 substitutes of each word in the user prompt.
    Furthermore, we use a fixed set of hyperparameters for both prompt improvement and adversarial attacks.
    We include a detailed discussion on all the hyperparameters and search space generation in Appendix~\ref{app:impl}.

\subsection{Application 1 - Adversarial Attack}    

\begin{table}[t]
    \centering
    \caption{Quantitative comparison of different prompting methods. We evaluate the generated images using both Spherical CLIP loss and Human Preference Score v2 (HPSv2) score (renormalized to 0-100) - a score trained to mimic human preferences on images generated from text. \textbf{Our method achieves the best result on both prompt improvement and adversarial attack among all methods}, including the previous SOTA - Promptist.}

    \resizebox{0.98\linewidth}{!}{
        \begin{tabular}{lcccccccc}
            \toprule
            \multirow{2}{*}{\textbf{Attack}}
            & \multicolumn{2}{c}{\textbf{DiffusionDB}} & \multicolumn{2}{c}{\textbf{COCO}} & \multicolumn{2}{c}{\textbf{ChatGPT}} \\
            & \textbf{CLIP$\uparrow$} & \textbf{HPSv2$\downarrow$} & \textbf{CLIP$\uparrow$} & \textbf{HPSv2$\downarrow$} & \textbf{CLIP$\uparrow$} & \textbf{HPSv2$\downarrow$} \\ \hline

            User   & 0.76 ± 0.03 & 75.28 ± 8.54 & 0.77 ± 0.03 & 75.28 ± 8.54 & 0.77 ± 0.02 & 73.57 ± 10.81 \\ \hline
            DPO-Diff     & \bf{0.86 ± 0.05} & \bf{40.52 ± 11.88} & \bf{0.94 ± 0.04} & \bf{45.85 ± 10.18} & \bf{0.95 ± 0.05} & \bf{39.73 ± 16.73} \\ \hline
        \end{tabular}
    }

    \resizebox{0.98\linewidth}{!}{
        \begin{tabular}{lcccccccccccc}
            \toprule
            \multirow{3}{*}{\textbf{Improve}}
            & \multicolumn{2}{c}{\textbf{DiffusionDB}} & \multicolumn{2}{c}{\textbf{COCO}} & \multicolumn{2}{c}{\textbf{ChatGPT}} \\
            & \textbf{CLIP$\downarrow$} & \textbf{HPSv2$\uparrow$} & \textbf{CLIP$\downarrow$} & \textbf{HPSv2$\uparrow$} & \textbf{CLIP$\downarrow$} & \textbf{HPSv2$\uparrow$} \\ \hline

            User            & 0.87 ± 0.02 & 48.81 ± 09.71 & 0.87 ± 0.01 & 50.33 ± 4.85 & 0.84 ± 0.01 & 53.36 ± 5.17 \\ \hline
            Manual          & 0.89 ± 0.04 & 51.43 ± 10.29 & - & - & - & - \\ \hline
            Promptist       & 0.88 ± 0.02 & 54.39 ± 12.47 & 0.87 ± 0.03 & 50.08 ± 7.43 & 0.85 ± 0.02 & 59.32 ± 6.50 \\ \hline
            DPO-Diff            &\bf{0.81 ± 0.03}&\bf{62.37 ± 12.48}&\bf{0.82 ± 0.02}&\bf{61.26 ± 0.77}&\bf{0.78 ± 0.03}&\bf{67.71 ± 6.46}\\ \hline
        \end{tabular}
    }

\label{tab:quantative}
\end{table}
    
    Unlike RLHF-based prompt-engineering methods (e.g. Promptist~\citep{promptist}) that require finetuning a prompt generator when adapting to a new task, DPO-Diff, as a train-free method, can be seamlessly applied to finding adversarial prompts by simply reversing the sign of the objective function.

    In this section, we demonstrate that DPO-Diff is capable of discovering adversarial prompts that destroy the prompt-following ability of Stable Diffusion.

\begin{figure*}[ht!]
    \vspace{+4mm}
    \centering
    \begin{tabularx}{\textwidth}{>{\centering\arraybackslash}X>{\centering\arraybackslash}X>{\centering\arraybackslash}X}
        \toprule\\[-10pt] {\textbf{User Input}} & {\textbf{Promptist - Modifiers}} & {\textbf{Negative Prompts by DPO-Diff}} \\[0pt]\hline

        \\[-12pt]
        {\prompt{The yellow sun was descending beyond the violet peaks, coloring the sky with hot shades.}} &
        {\prompt{by Greg Rutkowski and Raymond Swanland, ..., ultra realistic digital art}} &
        {\prompt{red, soaring, red, valleys, white, floor, Plain, body, focus, surreal}} \\[1pt]
        \includegraphics[width=0.8\linewidth]{figures/images/main_paper/yellow-sun-ori.pdf} &
        \includegraphics[width=0.8\linewidth]{figures/images/main_paper/yellow-sun-pts.pdf} &
        \includegraphics[width=0.8\linewidth]{figures/images/main_paper/yellow-sun-dpo.pdf}
        \\[-2pt]

        \\[-12pt]
        {\prompt{A dedicated \orange{gardener} tending to a ... \orange{bonsai tree}.}} &
        {\prompt{intricate, elegant, highly detailed, ..., sharp focus, illustration}} &
        {\prompt{irresponsible, overlooking, huge, herb, ...}} \\[1pt]
        \includegraphics[width=0.8\linewidth]{figures/images/main_paper/garden-ori.pdf} &
        \includegraphics[width=0.8\linewidth]{figures/images/main_paper/garden-pts.pdf} &
        \includegraphics[width=0.8\linewidth]{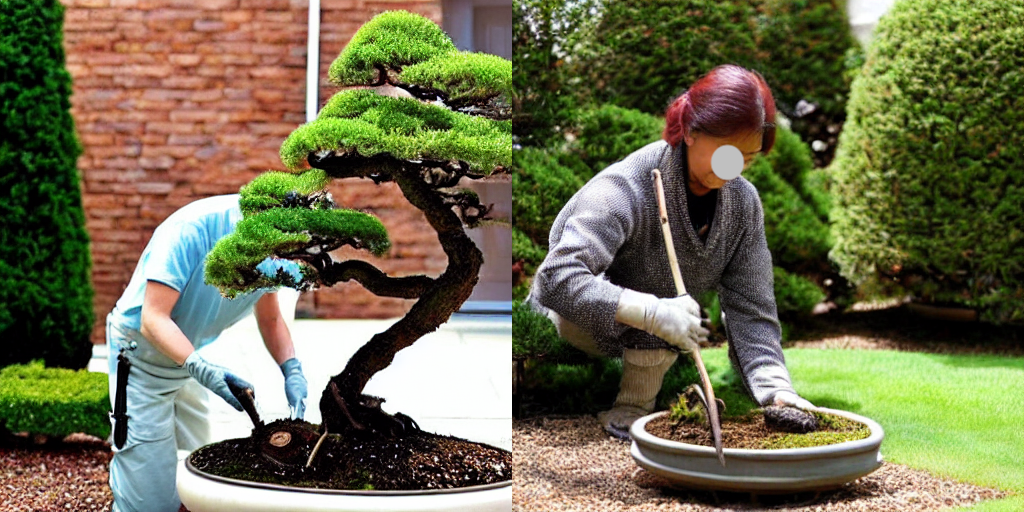}
        \\[-2pt]

        \\[-12pt]
        {\prompt{magical ... \orange{bear with glowing magical marks} ...}} &
        {\prompt{D\&D, fantasy, cinematic lighting, ..., art by artgerm and greg ...}} &
        {\prompt{normal, elephant, ..., heaps, tundra, advance, Boring, black, ...}} \\[1pt]
        \includegraphics[width=0.8\linewidth]{figures/images/main_paper/bear-ori.pdf} &
        \includegraphics[width=0.8\linewidth]{figures/images/main_paper/bear-pts.pdf} &
        \includegraphics[width=0.8\linewidth]{figures/images/main_paper/bear-dpo.pdf}
        \\[-2pt]


        \bottomrule

    \end{tabularx}

\caption{Example images generated by improved negative prompts from DPO-Diff v.s. Promptist (More in~\Cref{fig:improve_more}). Compared with Promptist, DPO-Diff was able to generate images that better capture the content in the original prompt.}
\vspace{+4mm}
\label{fig:improve}
\end{figure*}

\begin{figure*}[ht!]
\centering
    \begin{tabularx}{\textwidth}{>{\centering\arraybackslash}X>{\centering\arraybackslash}X}
        \toprule\\[-10pt] {\textbf{User Input}} & {\textbf{Adversarial Prompts by DPO-Diff}} \\[0pt]\hline

        \\[-12pt]
        {\prompt{A vibrant sunset casting hues of orange and pink.}} &
        {\prompt{The vibrant \orange{sundown} casting \orange{tones} of orange \orange{plus blush}.}} \\[2pt]
        \includegraphics[width=0.8\linewidth]{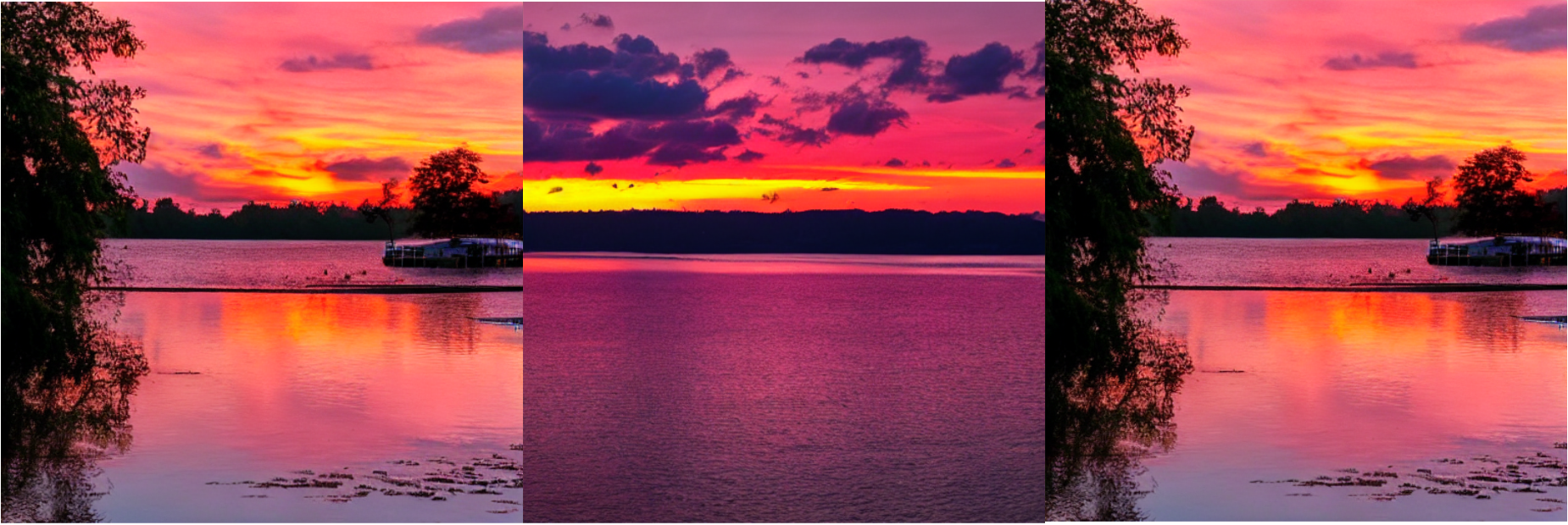} &
        \includegraphics[width=0.8\linewidth]{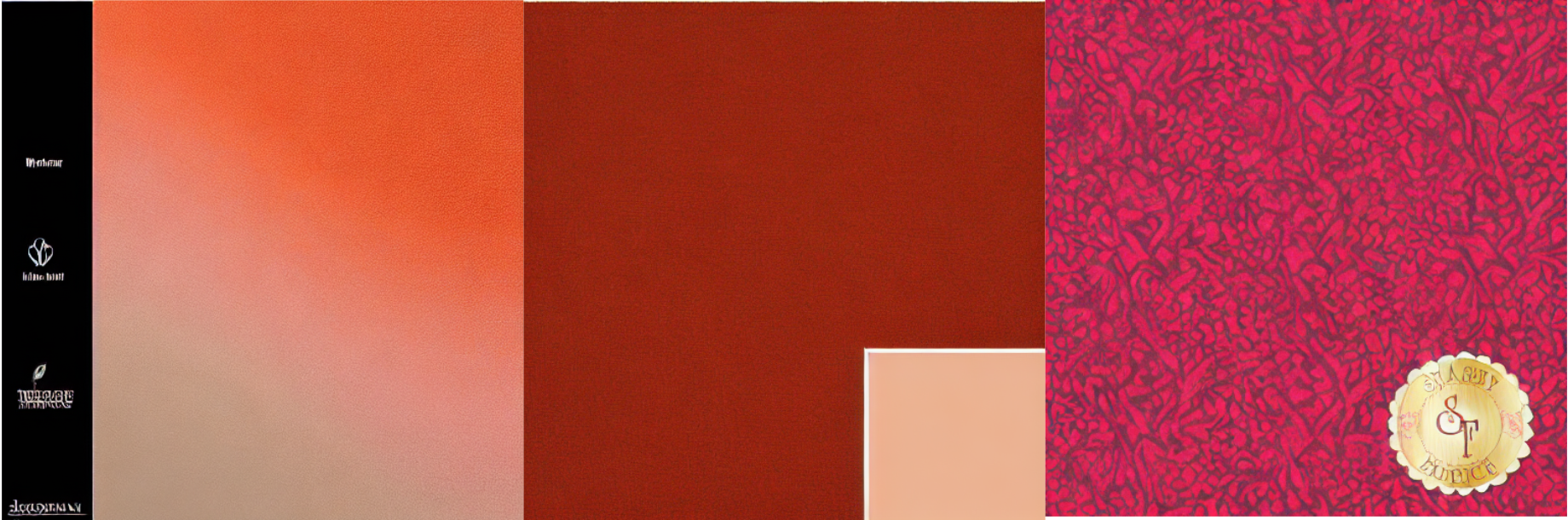}
        \\[-2pt] 

        \\[-12pt]
        {\prompt{A group of friends gather around a table for a meal.}} &
        {\prompt{A \orange{party} of friends \orange{cluster} around a \orange{surface} for a \orange{food}}} \\[2pt]
        \includegraphics[width=0.8\linewidth]{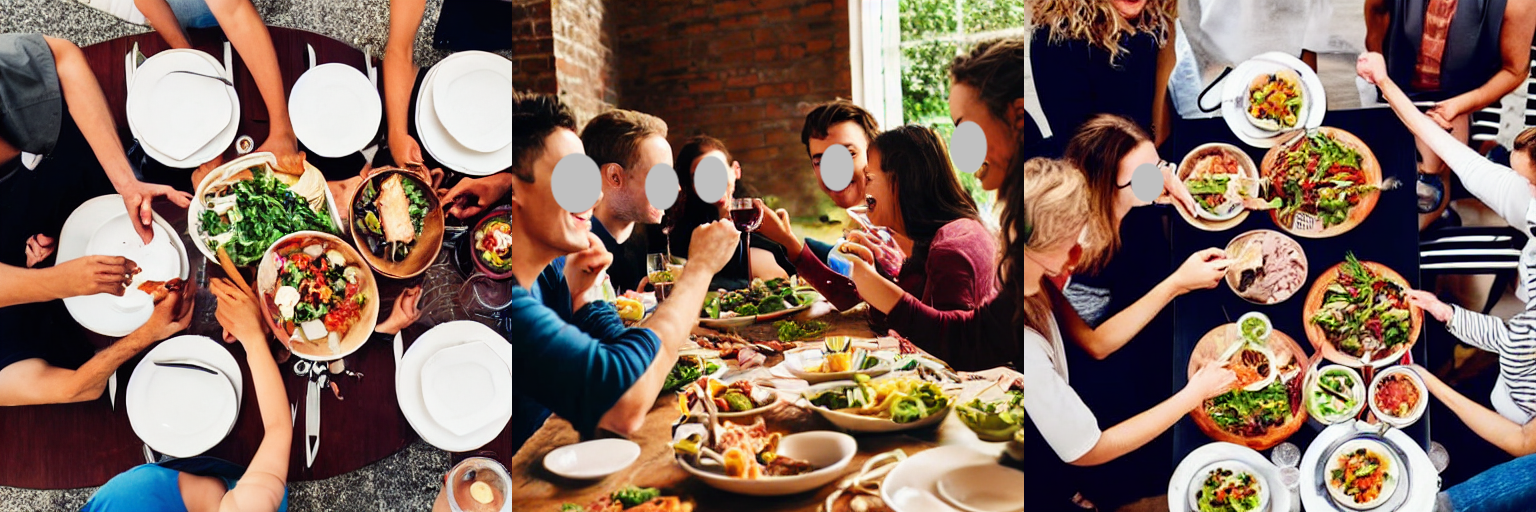} &
        \includegraphics[width=0.8\linewidth]{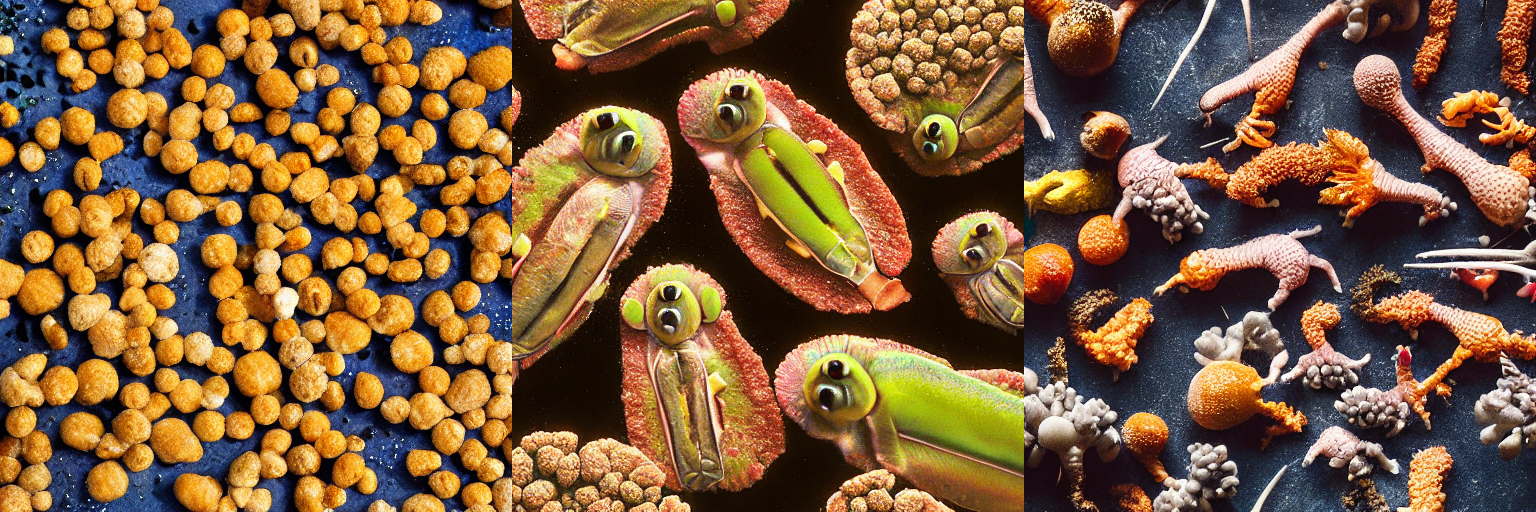}
        \\[-2pt] 

        \\[-12pt]
        {\prompt{oil painting of a mountain landscape}} &
        {\prompt{\orange{grease picture illustrating one} mountain \orange{view}}} \\[1pt]
        \includegraphics[width=0.8\linewidth]{figures/images/main_paper/picture-ori.pdf} &
        \includegraphics[width=0.8\linewidth]{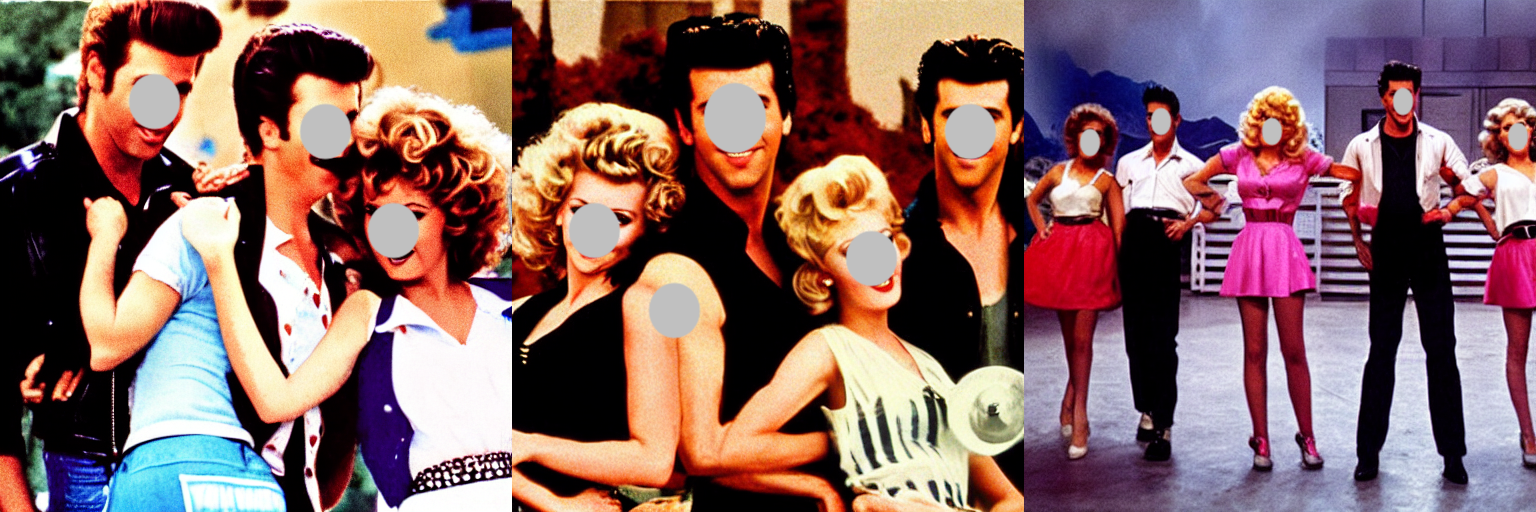}
        \\[-2pt]

        \bottomrule
    \end{tabularx}
    \caption{Example images generated by adversarial prompts from DPO-Diff. While keeping the overall meaning similar to the user input, adversarial prompts completely destroy the prompt-following ability of the Stable Diffusion model. (More in ~\Cref{fig:attack_more})}
\label{fig:attack}
\end{figure*}

    \uline{As suggested by \eqref{eq:attack}, a successful adversarial prompt must not change the original intention of the user prompt.}
    While we specified this constraint to ChatGPT when building the Synonyms Space, occasionally ChatGPT might mistake a word for the synonyms.
    To address this, during the evolutionary search phase, \textbf{we perform rejection sampling to refuse candidate prompts that have different meanings to the user input.}
    Concretely, we enforce their cosine similarity in embedding space to be higher than 0.9 (More on this can be found in \cref{app:exp_setting}).

    Table~\ref{tab:quantative} summarizes the quantitative results.
    Our method is able to perturb the original prompt to adversarial directions, resulting in a substantial increase in the clip loss.
    Figure \ref{fig:attack} also visualizes a set of intriguing images generated by the adversarial prompts.
    \textbf{We can see that DPO-Diff can effectively explore the text regions where Stable Diffusion fails to interpret.}

    \paragraph{Human Evaluation.}
    We further ask human judges to check whether the attack generated by DPO-Diff is successful or not. Since previous prompt optimization methods do not apply to this task, we only ask the evaluators to compare DPO-Diff against the original image. \textbf{DPO-Diff achieves an average success rate (ASR) of 44\% on SD-v1.}
    Considering that Stable Diffusion models are trained on a large amount of caption corpus, this success rate is fairly substantial.

\subsection{Application 2: Prompt Improvement}
    In this section, we apply DPO-Diff to craft prompts that improve the prompt-following ability of the generated images.
    We compare our method with three baselines:
    (1) User Input.
    (2) Human Engineered Prompts (available only on DiffusionDB)~\citep{diffusiondb}.
    (3) Promptist~\citep{promptist}, trained to mimic the human-crafted prompt provided in DiffusionDB.

    Table \ref{tab:quantative} summarizes the result.
    Among all methods, DPO-Diff achieves the best results under both Spherical CLIP loss and Human Preference Score (HPSv2) score.
    On the other hand, \uline{our findings suggest that both human-engineered and Promptist-optimized prompts do not improve the relevance between generated images and user intention}.
    The reason is that these methods merely add a set of aesthetic modifiers to the original prompt, irrelevant to the semantics of user input.
    This can be further observed from the qualitative examples in Figure~\ref{fig:improve}, where images generated by Promptist often also do not follow the prompts well.

    \paragraph{Human Evaluation.}
    We further ask human judges to rate DPO-Diff and Promptist on \uline{how well the generated images follow the user prompt.}
    Figure~\ref{fig:win_rate} summarizes the win/draw/loss rate of DPO-Diff against Promptist;
    \textbf{The result shows that DPO-Diff surpasses or matches Promptist in human rate 79\% of times on SD-v1.}

\subsection{Qualitative analysis of search progression}
 


\begin{figure}
\centering
    \begin{tabularx}{\linewidth}{>{\centering\arraybackslash}X}

        {\prompt{User Prompt: A bunch of luggage that is in front of a truck.}} \\[2pt]
        \includegraphics[width=0.18\linewidth]{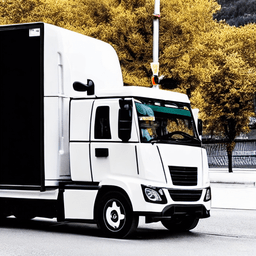}
        \includegraphics[width=0.18\linewidth]{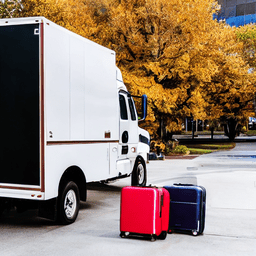}
        \includegraphics[width=0.18\linewidth]{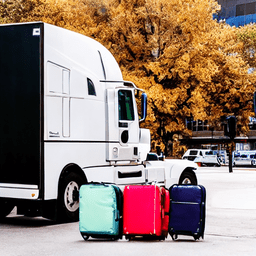}
        \includegraphics[width=0.18\linewidth]{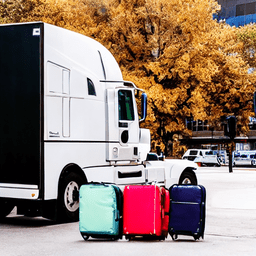}
        \includegraphics[width=0.18\linewidth]{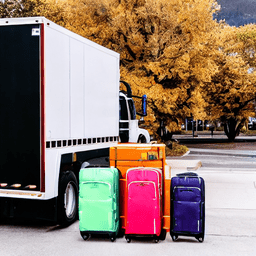}
        \\[-2pt] 

        {\prompt{User Prompt: There are cranes in the water and a boat in the distance.}} \\[2pt]
        \includegraphics[width=0.18\linewidth]{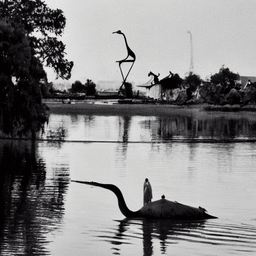}
        \includegraphics[width=0.18\linewidth]{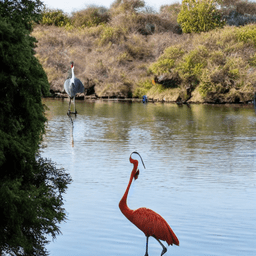}
        \includegraphics[width=0.18\linewidth]{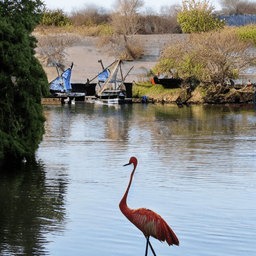}
        \includegraphics[width=0.18\linewidth]{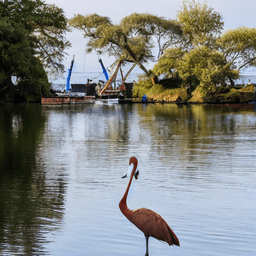}
        \includegraphics[width=0.18\linewidth]{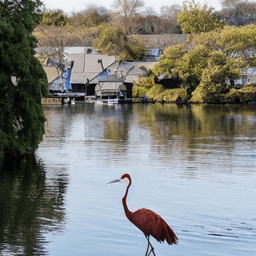}
        \\[-2pt] 

        {\prompt{User Prompt: harry potter shrek, movie poster, movie still, ...}} \\[2pt]
        \includegraphics[width=0.18\linewidth]{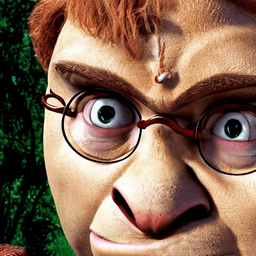}
        \includegraphics[width=0.18\linewidth]{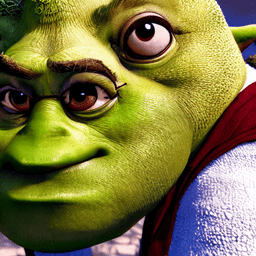}
        \includegraphics[width=0.18\linewidth]{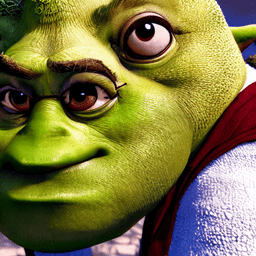}
        \includegraphics[width=0.18\linewidth]{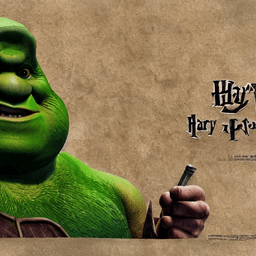}
        \includegraphics[width=0.18\linewidth]{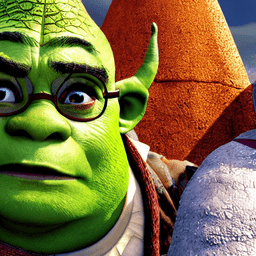}
        \\[-2pt] 

    \end{tabularx}
    \caption{Evolution of the optimized images from DPO-Diff at iteration 0, 10, 20, 40, and 80 (left to right). Noticeable improvements can be observed as early as 10 iterations, and the progression is surprisingly interpretable.}
\vspace{-4mm}
\label{fig:progression}
\end{figure}
To examine the convergence of our search algorithm qualitatively, we plot the progression of optimized images at various evaluation stages.
We set the target iterations at 0 (the original image), 10, 20, 40, and 80 to illustrate the changes, and showcase the image with the highest clip loss among all evaluated candidates at each iteration.

\Cref{fig:progression} illustrates some example trajectories.
In most cases, the images exhibit noticeable improvement in aligning with the user's prompt at as early as the 10th iteration, and continue to improve.
Moreover, the progression are surprisingly interpretable.
For instance, with the prompt: "\textbf{A bunch of luggage} in front of a truck," the initial image fails to include any luggage, featuring only the truck;
However, as the optimization continues, we can see that DPO-Diff incrementally adds more luggage to the scene.

\section{Ablation Study}

\begin{figure*}[t]
    \centering
 
    \includegraphics[width=0.48\linewidth]{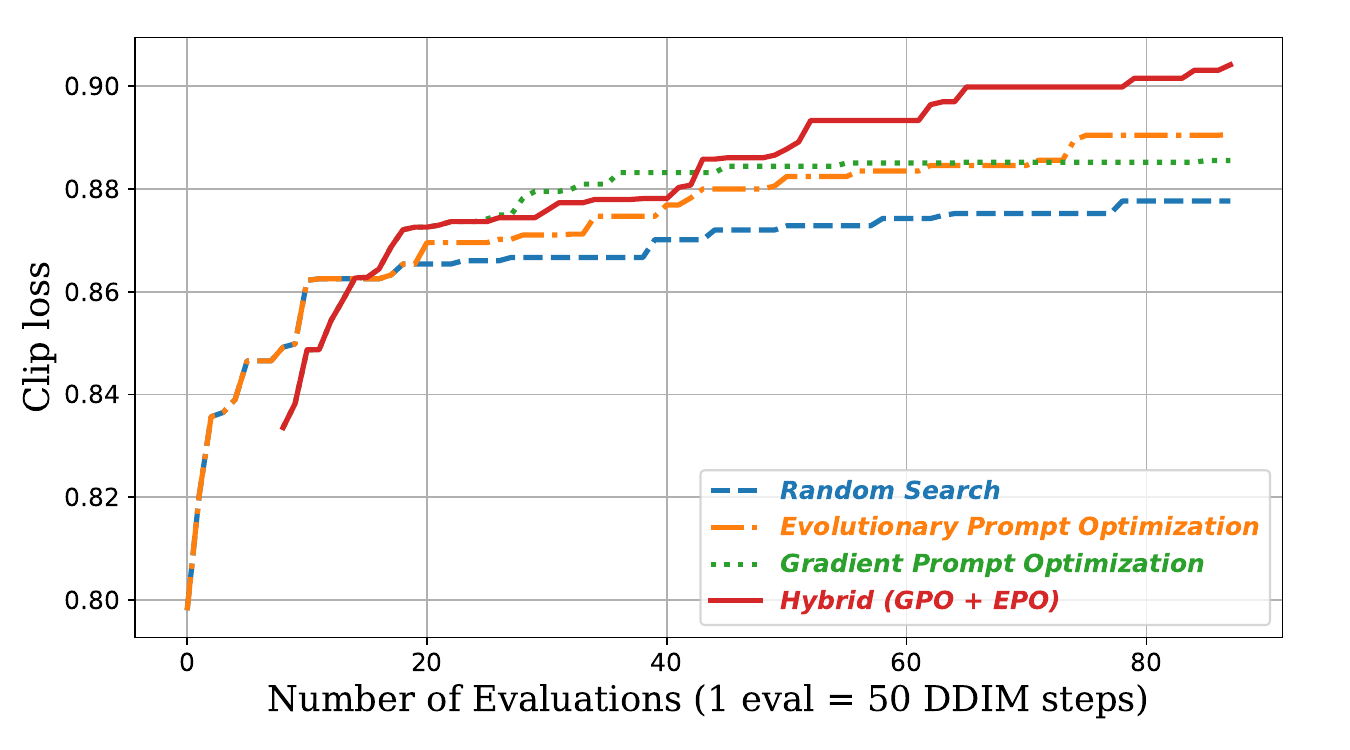}
    \label{fig:compare.attack}
    \includegraphics[width=0.48\linewidth]{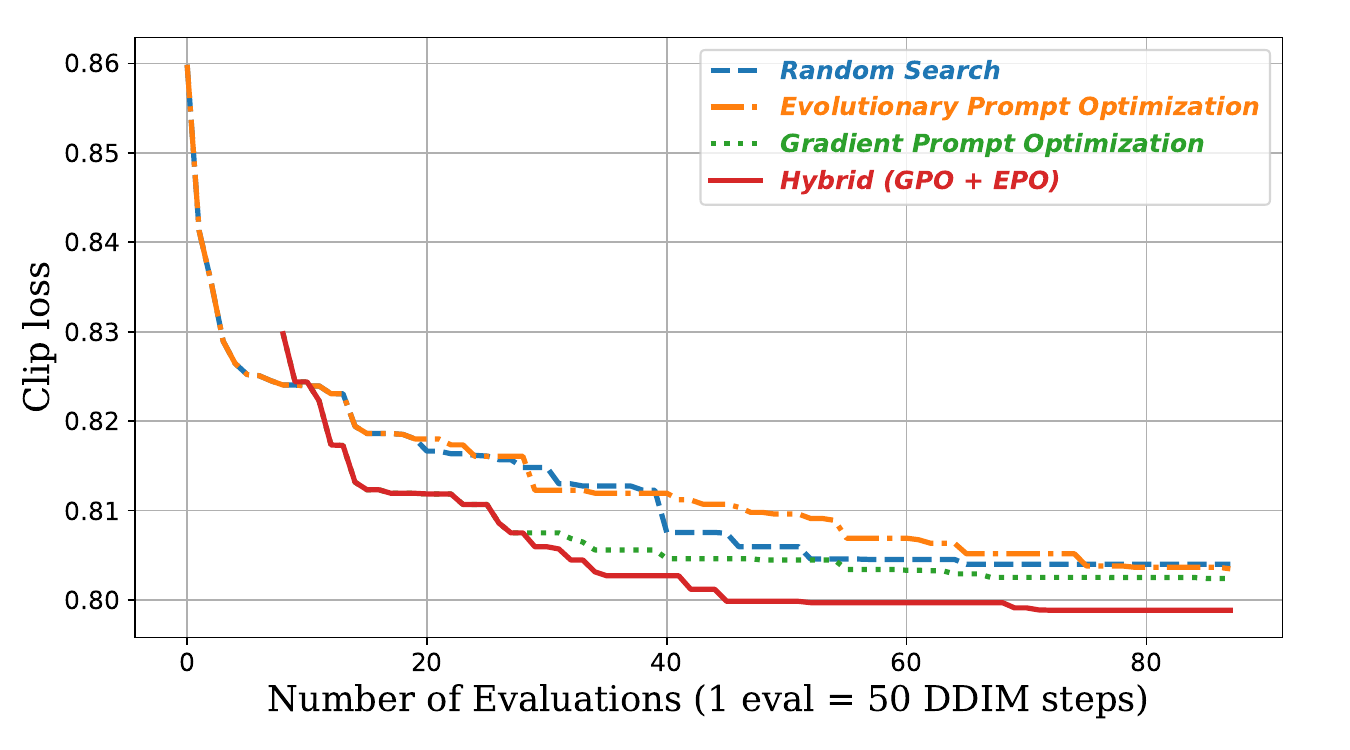}
    \label{fig:compare.improve}
    
    \caption{Learning curves of different search algorithms in solving DPO-Diff.}
    \label{fig:compare}
    
\end{figure*}

We conduct ablation studies on DPO-Diff using 30 randomly sampled prompts, 10 from each source.
Each search algorithm is run under 4 random seeds.

\subsection{Comparison of different search algorithms.}
\label{sec:ablate.algo}
We compare four search algorithms for DPO-Diff: Random Search (RS), Evolution Prompt Optimization (EPO), Gradient-based Prompt Optimization (GPO), and the full algorithm (GPO + ES).
Figure \ref{fig:compare} shows their performance under different search budgets (number of evaluations)\footnote{Since the runtime of backpropagation through one-step diffusion sampling is negligible w.r.t. the full sampling process (50 steps for DDIM sampler), we count it the same as one inference step.};
While GPO tops EPO under low budgets, it also plateaus quicker as randomly drawing from the learned distribution is sample-inefficient.
\textbf{Combining GPO with EPO achieves the best overall performance.}


\subsection{Negative prompt v.s. positive prompt optimization}
\label{sec:ablate.nvp}
One finding in our work is that optimizing negative prompts (Antonyms Space) is more effective than positive prompts (Synonyms Space) for Stable Diffusion.
To verify the strength of these spaces, we randomly sample 100 prompts for each space and compute their average clip loss of generated images.
Table \ref{tab:ablate.nvp} suggests that Antonyms Space contains candidates with consistently lower clip loss than Synonyms Space.

\begin{table}[t]
    \centering
    \caption{Quantitative evaluation of optimizing negative prompts (w/ Antonyms Space) and positive prompts (w/ Synonym Space) for Stable Diffusion.}

    \resizebox{0.9\linewidth}{!}{
        \begin{tabular}{lccccccc}
            \hline
            \textbf{Prompt} & \textbf{DiffusionDB} & \textbf{ChatGPT}     & \textbf{COCO} \\ \hline
            User Input      & 0.8741 ± 0.0203      & 0.8159 ± 0.0100      & 0.8606 ± 0.0096 \\ 
            Positive Prompt & 0.8747 ± 0.0189      & 0.8304 ± 0.0284      & 0.8624 ± 0.0141 \\ 
            Negative Prompt & \bf{0.8579 ± 0.0242} & \bf{0.8133 ± 0.0197} & \bf{0.8403 ± 0.0210} \\ \hline
            \end{tabular}
    }
\label{tab:ablate.nvp}
\end{table}
\section{Discussion on the Search v.s. Learning paradigms for utilizing computatons}
\label{sec:method.comparison}
This section elucidates the relationship between two distinct prompt optimization approaches for diffusion models: DPO-Diff (ours) and Promptist.
While Promptist represents a pioneering effort, it is important to discuss why DPO-Diff remains essential.

\paragraph{Limitations of Promptist}
Promptist utilizes the Reinforcement Learning from Human Feedback (RLHF)~\citep{rlhf1, rlhf2, chatgpt} approach to fine-tune a language model to generate improved prompts.
RLHF relies on paired data $\langle\texttt{user\_prompt}, \texttt{improved\_prompt}\rangle$, which is scarce for diffusion models and challenging to curate.
This is primarily because generating the improved prompts requires extensive trial-and-error by human experts, essentially performing what DPO-Diff automates.
In fact, the performance limit exhibited by Promptist is exactly caused by this lack of data:
The data used by Promptist from DiffusionDB predominantly features aesthetic modifiers that do not alter the semantics of the prompts
This limits its effectiveness to aesthetic enhancements and not addressing the core need for semantic accuracy in prompts. 
Consequently, it struggles with semantic prompt adherence and lacks flexibility in modifying prompts for tasks such as adversarial attacks.

\paragraph{Two complementary computational paradigms}
Promptist and DPO-Diff represent two major paradigms for effectively utilizing computation: learning and searching, respectively~\citep{lesson}.
Learning-based approach of Promptist enhances performance through more parameters and larger datasets, whereas the search-based approach of DPO-Diff focuses on maximizing the potential of pretrained models via post-hoc optimization.
Although learning-based methods require high quality paired data, they can be efficiently deployed once trained;
On the other hand, search-based methods generate high quality prompts, but are much slower to execute.
Therefore, as~\citet{lesson} highlights, these paradigms are complementary rather than competitive.
DPO-Diff can be leveraged to generate high quality dataset offline, which can subsequently train Promptist to reduce inference latency effectively.
Together, they pave the way for a comprehensive solution to prompt optimization for diffusion models, positioning DPO-Diff as the first search-based solution to address this problem.

\section{Conclusions}
\label{sec:conclusion}
This work presents DPO-Diff, the first gradient-based framework for optimizing discrete prompts.
We formulate prompt optimization as a discrete optimization problem over the text space.
To improve the search efficiency, we introduce a family of compact search spaces based on relevant word substitutions, as well as design a generic computational method for computing the discrete text gradient for diffusion model's inference process.
DPO-Diff is generic - We demonstrate that it can be directly applied to effectively discover both refined prompts to aid image generation and adversarial prompts for model diagnosis.
We hope that the proposed framework helps open up new possibilities in developing advanced prompt optimization methods for text-based image generation tasks.
\paragraph{Limitations} To motivate future work, we discuss the known limitations of DPO-Diff in Appendix~\ref{app:limitation}.

\newpage
\section*{Acknowledgements}
The work is partially supported by NSF 2048280, 2331966, 2325121, 2244760, ONR N00014-23-1-2300, and finished during the primary contributor's internship at Google.
Special thanks to Liangzhe Yuan, Long Zhao, and Han Zhang for providing invaluable guidance and accommodations throughout the internship.

\section*{Impact Statement}
This work makes contribution to both research and practical applications of text-to-image (T2I) generation.
For the research community, we introduce a new paradigm to optimize prompts for text-to-image generation, demonstrating promising results across various prompts, models, and metrics.
This approach could provide valuable insights for future studies on diffusion models.
For industrial applications, our method can be easily adopted by T2I generation service providers to improve the performance of their models, or used as an offline data generator for training prompt agents.

\bibliography{example_paper}

\begin{thebibliography}{53}
\providecommand{\natexlab}[1]{#1}
\providecommand{\url}[1]{\texttt{#1}}
\expandafter\ifx\csname urlstyle\endcsname\relax
  \providecommand{\doi}[1]{doi: #1}\else
  \providecommand{\doi}{doi: \begingroup \urlstyle{rm}\Url}\fi

\bibitem[Alzantot et~al.(2018)Alzantot, Sharma, Elgohary, Ho, Srivastava, and Chang]{synonym1}
Alzantot, M., Sharma, Y., Elgohary, A., Ho, B.-J., Srivastava, M., and Chang, K.-W.
\newblock Generating natural language adversarial examples.
\newblock \emph{arXiv preprint arXiv:1804.07998}, 2018.

\bibitem[Andrew(2023)]{np1}
Andrew.
\newblock How to use negative prompts?, 2023.
\newblock URL \url{https://lexica.art/}.

\bibitem[Art(Year)]{lexica}
Art, L.
\newblock Lexica, Year.
\newblock URL \url{https://lexica.art/}.

\bibitem[Bain \& Sammut(1995)Bain and Sammut]{rlhf1}
Bain, M. and Sammut, C.
\newblock A framework for behavioural cloning.
\newblock In \emph{Machine Intelligence 15}, pp.\  103--129, 1995.

\bibitem[Chang et~al.(2023)Chang, Zhang, Barber, Maschinot, Lezama, Jiang, Yang, Murphy, Freeman, Rubinstein, et~al.]{muse}
Chang, H., Zhang, H., Barber, J., Maschinot, A., Lezama, J., Jiang, L., Yang, M.-H., Murphy, K., Freeman, W.~T., Rubinstein, M., et~al.
\newblock Muse: Text-to-image generation via masked generative transformers.
\newblock \emph{arXiv preprint arXiv:2301.00704}, 2023.

\bibitem[Chen et~al.(2023)Chen, Chen, Goldstein, Huang, and Zhou]{instructzero}
Chen, L., Chen, J., Goldstein, T., Huang, H., and Zhou, T.
\newblock Instructzero: Efficient instruction optimization for black-box large language models.
\newblock \emph{arXiv preprint arXiv:2306.03082}, 2023.

\bibitem[Cheng et~al.(2018)Cheng, Le, Chen, Yi, Zhang, and Hsieh]{opt}
Cheng, M., Le, T., Chen, P.-Y., Yi, J., Zhang, H., and Hsieh, C.-J.
\newblock Query-efficient hard-label black-box attack: An optimization-based approach.
\newblock \emph{arXiv preprint arXiv:1807.04457}, 2018.

\bibitem[Cherti et~al.(2023)Cherti, Beaumont, Wightman, Wortsman, Ilharco, Gordon, Schuhmann, Schmidt, and Jitsev]{openclip}
Cherti, M., Beaumont, R., Wightman, R., Wortsman, M., Ilharco, G., Gordon, C., Schuhmann, C., Schmidt, L., and Jitsev, J.
\newblock Reproducible scaling laws for contrastive language-image learning.
\newblock In \emph{Proceedings of the IEEE/CVF Conference on Computer Vision and Pattern Recognition}, pp.\  2818--2829, 2023.

\bibitem[Christiano et~al.(2017)Christiano, Leike, Brown, Martic, Legg, and Amodei]{rlhf2}
Christiano, P.~F., Leike, J., Brown, T., Martic, M., Legg, S., and Amodei, D.
\newblock Deep reinforcement learning from human preferences.
\newblock \emph{Advances in neural information processing systems}, 30, 2017.

\bibitem[Crowson et~al.(2022)Crowson, Biderman, Kornis, Stander, Hallahan, Castricato, and Raff]{vqgan-clip}
Crowson, K., Biderman, S., Kornis, D., Stander, D., Hallahan, E., Castricato, L., and Raff, E.
\newblock Vqgan-clip: Open domain image generation and editing with natural language guidance.
\newblock In \emph{European Conference on Computer Vision}, pp.\  88--105. Springer, 2022.

\bibitem[Crumb(2022)]{spherical}
Crumb.
\newblock Clip-guided stable diffusion, 2022.
\newblock URL \url{https://crumbly.medium.com/}.

\bibitem[Dale(2021)]{gpt3}
Dale, R.
\newblock Gpt-3: What’s it good for?
\newblock \emph{Natural Language Engineering}, 27\penalty0 (1):\penalty0 113--118, 2021.

\bibitem[Dhariwal \& Nichol(2021)Dhariwal and Nichol]{apm}
Dhariwal, P. and Nichol, A.
\newblock Diffusion models beat gans on image synthesis.
\newblock \emph{Advances in neural information processing systems}, 34:\penalty0 8780--8794, 2021.

\bibitem[Dong \& Yang(2019)Dong and Yang]{gdas}
Dong, X. and Yang, Y.
\newblock Searching for a robust neural architecture in four gpu hours.
\newblock In \emph{Proceedings of the IEEE/CVF Conference on Computer Vision and Pattern Recognition}, pp.\  1761--1770, 2019.

\bibitem[Feng et~al.(2022)Feng, He, Fu, Jampani, Akula, Narayana, Basu, Wang, and Wang]{structdiff}
Feng, W., He, X., Fu, T.-J., Jampani, V., Akula, A., Narayana, P., Basu, S., Wang, X.~E., and Wang, W.~Y.
\newblock Training-free structured diffusion guidance for compositional text-to-image synthesis.
\newblock \emph{arXiv preprint arXiv:2212.05032}, 2022.

\bibitem[Gal et~al.(2022)Gal, Alaluf, Atzmon, Patashnik, Bermano, Chechik, and Cohen-Or]{ti}
Gal, R., Alaluf, Y., Atzmon, Y., Patashnik, O., Bermano, A.~H., Chechik, G., and Cohen-Or, D.
\newblock An image is worth one word: Personalizing text-to-image generation using textual inversion.
\newblock \emph{arXiv preprint arXiv:2208.01618}, 2022.

\bibitem[Goldberg(1989)]{ea}
Goldberg, D.~E.
\newblock \emph{Genetic Algorithms in Search, Optimization and Machine Learning 1st Edition}.
\newblock Addison-Wesley Professional, 1989.
\newblock ISBN 978-0201157673.

\bibitem[Guo et~al.(2021)Guo, Sablayrolles, J{\'e}gou, and Kiela]{gbda}
Guo, C., Sablayrolles, A., J{\'e}gou, H., and Kiela, D.
\newblock Gradient-based adversarial attacks against text transformers.
\newblock \emph{arXiv preprint arXiv:2104.13733}, 2021.

\bibitem[Guo et~al.(2023)Guo, Wang, Guo, Li, Song, Tan, Liu, Bian, and Yang]{prompt-evo}
Guo, Q., Wang, R., Guo, J., Li, B., Song, K., Tan, X., Liu, G., Bian, J., and Yang, Y.
\newblock Connecting large language models with evolutionary algorithms yields powerful prompt optimizers.
\newblock \emph{arXiv preprint arXiv:2309.08532}, 2023.

\bibitem[Hao et~al.(2022)Hao, Chi, Dong, and Wei]{promptist}
Hao, Y., Chi, Z., Dong, L., and Wei, F.
\newblock Optimizing prompts for text-to-image generation.
\newblock \emph{arXiv preprint arXiv:2212.09611}, 2022.

\bibitem[Ho \& Salimans(2022)Ho and Salimans]{cfg}
Ho, J. and Salimans, T.
\newblock Classifier-free diffusion guidance.
\newblock \emph{arXiv preprint arXiv:2207.12598}, 2022.

\bibitem[Ho et~al.(2020)Ho, Jain, and Abbeel]{ddpm}
Ho, J., Jain, A., and Abbeel, P.
\newblock Denoising diffusion probabilistic models.
\newblock \emph{Advances in neural information processing systems}, 33:\penalty0 6840--6851, 2020.

\bibitem[Ho et~al.(2022)Ho, Chan, Saharia, Whang, Gao, Gritsenko, Kingma, Poole, Norouzi, Fleet, et~al.]{imagen-video}
Ho, J., Chan, W., Saharia, C., Whang, J., Gao, R., Gritsenko, A., Kingma, D.~P., Poole, B., Norouzi, M., Fleet, D.~J., et~al.
\newblock Imagen video: High definition video generation with diffusion models.
\newblock \emph{arXiv preprint arXiv:2210.02303}, 2022.

\bibitem[Ilyas et~al.(2018)Ilyas, Engstrom, Athalye, and Lin]{blackbox}
Ilyas, A., Engstrom, L., Athalye, A., and Lin, J.
\newblock Black-box adversarial attacks with limited queries and information.
\newblock In \emph{International conference on machine learning}, pp.\  2137--2146. PMLR, 2018.

\bibitem[Jang et~al.(2016)Jang, Gu, and Poole]{gumbel}
Jang, E., Gu, S., and Poole, B.
\newblock Categorical reparameterization with gumbel-softmax.
\newblock \emph{arXiv preprint arXiv:1611.01144}, 2016.

\bibitem[Lian et~al.(2023)Lian, Li, Yala, and Darrell]{grounded}
Lian, L., Li, B., Yala, A., and Darrell, T.
\newblock Llm-grounded diffusion: Enhancing prompt understanding of text-to-image diffusion models with large language models.
\newblock \emph{arXiv preprint arXiv:2305.13655}, 2023.

\bibitem[Lin et~al.(2014)Lin, Maire, Belongie, Hays, Perona, Ramanan, Doll{\'a}r, and Zitnick]{coco}
Lin, T.-Y., Maire, M., Belongie, S., Hays, J., Perona, P., Ramanan, D., Doll{\'a}r, P., and Zitnick, C.~L.
\newblock Microsoft coco: Common objects in context.
\newblock In \emph{Computer Vision--ECCV 2014: 13th European Conference, Zurich, Switzerland, September 6-12, 2014, Proceedings, Part V 13}, pp.\  740--755. Springer, 2014.

\bibitem[Liu et~al.(2022)Liu, Li, Du, Torralba, and Tenenbaum]{compdiff}
Liu, N., Li, S., Du, Y., Torralba, A., and Tenenbaum, J.~B.
\newblock Compositional visual generation with composable diffusion models.
\newblock In \emph{European Conference on Computer Vision}, pp.\  423--439. Springer, 2022.

\bibitem[Liu \& Chilton(2022)Liu and Chilton]{pe-diff2}
Liu, V. and Chilton, L.~B.
\newblock Design guidelines for prompt engineering text-to-image generative models.
\newblock In \emph{Proceedings of the 2022 CHI Conference on Human Factors in Computing Systems}, pp.\  1--23, 2022.

\bibitem[Mokady et~al.(2023)Mokady, Hertz, Aberman, Pritch, and Cohen-Or]{nti}
Mokady, R., Hertz, A., Aberman, K., Pritch, Y., and Cohen-Or, D.
\newblock Null-text inversion for editing real images using guided diffusion models.
\newblock In \emph{Proceedings of the IEEE/CVF Conference on Computer Vision and Pattern Recognition}, pp.\  6038--6047, 2023.

\bibitem[Nie et~al.(2022)Nie, Guo, Huang, Xiao, Vahdat, and Anandkumar]{bp2}
Nie, W., Guo, B., Huang, Y., Xiao, C., Vahdat, A., and Anandkumar, A.
\newblock Diffusion models for adversarial purification.
\newblock \emph{arXiv preprint arXiv:2205.07460}, 2022.

\bibitem[Ouyang et~al.(2022)Ouyang, Wu, Jiang, Almeida, Wainwright, Mishkin, Zhang, Agarwal, Slama, Ray, et~al.]{chatgpt}
Ouyang, L., Wu, J., Jiang, X., Almeida, D., Wainwright, C., Mishkin, P., Zhang, C., Agarwal, S., Slama, K., Ray, A., et~al.
\newblock Training language models to follow instructions with human feedback.
\newblock \emph{Advances in Neural Information Processing Systems}, 35:\penalty0 27730--27744, 2022.

\bibitem[Pryzant et~al.(2023)Pryzant, Iter, Li, Lee, Zhu, and Zeng]{apo}
Pryzant, R., Iter, D., Li, J., Lee, Y.~T., Zhu, C., and Zeng, M.
\newblock Automatic prompt optimization with" gradient descent" and beam search.
\newblock \emph{arXiv preprint arXiv:2305.03495}, 2023.

\bibitem[Radford et~al.(2021)Radford, Kim, Hallacy, Ramesh, Goh, Agarwal, Sastry, Askell, Mishkin, Clark, et~al.]{clip}
Radford, A., Kim, J.~W., Hallacy, C., Ramesh, A., Goh, G., Agarwal, S., Sastry, G., Askell, A., Mishkin, P., Clark, J., et~al.
\newblock Learning transferable visual models from natural language supervision.
\newblock In \emph{International conference on machine learning}, pp.\  8748--8763. PMLR, 2021.

\bibitem[Raffel et~al.(2020)Raffel, Shazeer, Roberts, Lee, Narang, Matena, Zhou, Li, and Liu]{t5}
Raffel, C., Shazeer, N., Roberts, A., Lee, K., Narang, S., Matena, M., Zhou, Y., Li, W., and Liu, P.~J.
\newblock Exploring the limits of transfer learning with a unified text-to-text transformer.
\newblock \emph{The Journal of Machine Learning Research}, 21\penalty0 (1):\penalty0 5485--5551, 2020.

\bibitem[Ramesh et~al.(2022)Ramesh, Dhariwal, Nichol, Chu, and Chen]{dalle2}
Ramesh, A., Dhariwal, P., Nichol, A., Chu, C., and Chen, M.
\newblock Hierarchical text-conditional image generation with clip latents.
\newblock \emph{arXiv preprint arXiv:2204.06125}, 1\penalty0 (2):\penalty0 3, 2022.

\bibitem[Rombach et~al.(2022)Rombach, Blattmann, Lorenz, Esser, and Ommer]{sd}
Rombach, R., Blattmann, A., Lorenz, D., Esser, P., and Ommer, B.
\newblock High-resolution image synthesis with latent diffusion models.
\newblock In \emph{Proceedings of the IEEE/CVF conference on computer vision and pattern recognition}, pp.\  10684--10695, 2022.

\bibitem[Ronneberger et~al.(2015)Ronneberger, Fischer, and Brox]{unet}
Ronneberger, O., Fischer, P., and Brox, T.
\newblock U-net: Convolutional networks for biomedical image segmentation.
\newblock In \emph{Medical Image Computing and Computer-Assisted Intervention--MICCAI 2015: 18th International Conference, Munich, Germany, October 5-9, 2015, Proceedings, Part III 18}, pp.\  234--241. Springer, 2015.

\bibitem[Saharia et~al.(2022)Saharia, Chan, Saxena, Li, Whang, Denton, Ghasemipour, Gontijo~Lopes, Karagol~Ayan, Salimans, et~al.]{imagen}
Saharia, C., Chan, W., Saxena, S., Li, L., Whang, J., Denton, E.~L., Ghasemipour, K., Gontijo~Lopes, R., Karagol~Ayan, B., Salimans, T., et~al.
\newblock Photorealistic text-to-image diffusion models with deep language understanding.
\newblock \emph{Advances in Neural Information Processing Systems}, 35:\penalty0 36479--36494, 2022.

\bibitem[S{\o}nderby et~al.(2016)S{\o}nderby, Raiko, Maal{\o}e, S{\o}nderby, and Winther]{hvae}
S{\o}nderby, C.~K., Raiko, T., Maal{\o}e, L., S{\o}nderby, S.~K., and Winther, O.
\newblock Ladder variational autoencoders.
\newblock \emph{Advances in neural information processing systems}, 29, 2016.

\bibitem[Song et~al.(2020)Song, Meng, and Ermon]{ddim}
Song, J., Meng, C., and Ermon, S.
\newblock Denoising diffusion implicit models.
\newblock \emph{arXiv preprint arXiv:2010.02502}, 2020.

\bibitem[Sutton(2019)]{lesson}
Sutton, R.
\newblock The bitter lesson.
\newblock \emph{Incomplete Ideas (blog)}, 13\penalty0 (1):\penalty0 38, 2019.

\bibitem[Touvron et~al.(2023)Touvron, Lavril, Izacard, Martinet, Lachaux, Lacroix, Rozi{\`e}re, Goyal, Hambro, Azhar, et~al.]{llama}
Touvron, H., Lavril, T., Izacard, G., Martinet, X., Lachaux, M.-A., Lacroix, T., Rozi{\`e}re, B., Goyal, N., Hambro, E., Azhar, F., et~al.
\newblock Llama: Open and efficient foundation language models.
\newblock \emph{arXiv preprint arXiv:2302.13971}, 2023.

\bibitem[Wang et~al.(2022)Wang, Montoya, Munechika, Yang, Hoover, and Chau]{diffusiondb}
Wang, Z.~J., Montoya, E., Munechika, D., Yang, H., Hoover, B., and Chau, D.~H.
\newblock Diffusiondb: A large-scale prompt gallery dataset for text-to-image generative models.
\newblock \emph{arXiv preprint arXiv:2210.14896}, 2022.

\bibitem[Watson et~al.(2021)Watson, Chan, Ho, and Norouzi]{bp1}
Watson, D., Chan, W., Ho, J., and Norouzi, M.
\newblock Learning fast samplers for diffusion models by differentiating through sample quality.
\newblock In \emph{International Conference on Learning Representations}, 2021.

\bibitem[Wen et~al.(2023)Wen, Jain, Kirchenbauer, Goldblum, Geiping, and Goldstein]{pez}
Wen, Y., Jain, N., Kirchenbauer, J., Goldblum, M., Geiping, J., and Goldstein, T.
\newblock Hard prompts made easy: Gradient-based discrete optimization for prompt tuning and discovery.
\newblock \emph{arXiv preprint arXiv:2302.03668}, 2023.

\bibitem[Witteveen \& Andrews(2022)Witteveen and Andrews]{pe-diff}
Witteveen, S. and Andrews, M.
\newblock Investigating prompt engineering in diffusion models.
\newblock \emph{arXiv preprint arXiv:2211.15462}, 2022.

\bibitem[Woolf(2022)]{np2}
Woolf, M.
\newblock Lexica, 2022.
\newblock URL \url{https://minimaxir.com/2022/11/stable-diffusion-negative-prompt/}.

\bibitem[Wu et~al.(2019)Wu, Dai, Zhang, Wang, Sun, Wu, Tian, Vajda, Jia, and Keutzer]{fbnet}
Wu, B., Dai, X., Zhang, P., Wang, Y., Sun, F., Wu, Y., Tian, Y., Vajda, P., Jia, Y., and Keutzer, K.
\newblock Fbnet: Hardware-aware efficient convnet design via differentiable neural architecture search.
\newblock In \emph{Proceedings of the IEEE/CVF conference on computer vision and pattern recognition}, pp.\  10734--10742, 2019.

\bibitem[Xu et~al.(2024)Xu, Liu, Wu, Tong, Li, Ding, Tang, and Dong]{refl}
Xu, J., Liu, X., Wu, Y., Tong, Y., Li, Q., Ding, M., Tang, J., and Dong, Y.
\newblock Imagereward: Learning and evaluating human preferences for text-to-image generation.
\newblock \emph{Advances in Neural Information Processing Systems}, 36, 2024.

\bibitem[Yang et~al.(2023)Yang, Wang, Lu, Liu, Le, Zhou, and Chen]{llm-opt}
Yang, C., Wang, X., Lu, Y., Liu, H., Le, Q.~V., Zhou, D., and Chen, X.
\newblock Large language models as optimizers.
\newblock \emph{arXiv preprint arXiv:2309.03409}, 2023.

\bibitem[Yu et~al.(2022)Yu, Xu, Koh, Luong, Baid, Wang, Vasudevan, Ku, Yang, Ayan, et~al.]{parti}
Yu, J., Xu, Y., Koh, J.~Y., Luong, T., Baid, G., Wang, Z., Vasudevan, V., Ku, A., Yang, Y., Ayan, B.~K., et~al.
\newblock Scaling autoregressive models for content-rich text-to-image generation.
\newblock \emph{arXiv preprint arXiv:2206.10789}, 2\penalty0 (3):\penalty0 5, 2022.

\bibitem[Zhou et~al.(2022)Zhou, Muresanu, Han, Paster, Pitis, Chan, and Ba]{ape}
Zhou, Y., Muresanu, A.~I., Han, Z., Paster, K., Pitis, S., Chan, H., and Ba, J.
\newblock Large language models are human-level prompt engineers.
\newblock \emph{arXiv preprint arXiv:2211.01910}, 2022.

\end{thebibliography}
\bibliographystyle{icml2024}

\newpage
\appendix
\label{app}
\onecolumn

\section{Limitations}
\label{app:limitation}
We identify the following known limitations of the proposed method:
\textbf{Search cost}
Our method requires multiple passes through the diffusion model to optimize a given prompt, which incurs a modest amount of search costs.
One promising solution is to use DPO-Diff to generate free paired data for RLHF (e.g. Promptist), which we leave for future work to explore.
\textbf{Text encoder}
moreover, while DPO-Diff improves the faithfulness of the generated image, the performance is upper-bounded by the limitations of the underlying text encoder.
For example, the clip text encoder used in stable diffusion tends to discard spatial relationships in text, which in principle must be resolved by improving the model itself, such as augmenting the diffusion model with a powerful LLM~\cite{grounded, compdiff, structdiff}.
\textbf{Clip loss}
The clip loss used in DPO-Diff might not always align with human evaluation.
Automatic scoring metrics that better reflect human judgment, similar to the reward models used in instruction fine-tuning, can further aid the discovery of improved prompts.
\textbf{Synonyms generated by ChatGPT}
For adversarial attack task, ChatGPT sometimes generate incorrect synonyms.
Although we use reject-sampling based on sentence embedding similarity as a posthoc fix, it is not completely accurate.
This may impact the validity of adversarial prompts, as by definition they must preserve the user's original intent.
We address this in human evaluation by asking the raters to consider this factor when determining the success of an attack.

\section{Benefit of optimizing discrete text prompts over soft prompts}
Optimizing discrete text prompts offers two major advantages over tuning soft prompts, primarily in two areas:
\textbf{(1) Interpretability:} The results of discrete prompt optimization are texts that are naturally human interpretable.
This also facilitates direct use in fine-tuning RLHF-based agents like Promptist.
\textbf{(2) Simplified Search Space:} Our preliminary attempts with continuous text embeddings revealed challenges in achieving convergence, even on toy examples.
The reason, we conjecture was that the gradients backpropagated through the denoising process have low info-to-noise ratio; And updating soft prompt using such gradient could be very unstable due to its huge continuous search space.
In contrast, discrete prompt optimization effectively narrows the search to a finite vocabulary set, greatly reducing search complexity and improving stability.

\section{Derivation for the alternative interpretation of DDPM's modeling.}
\label{app:theory}
\begin{proposition}
    The original parameterization of DDPM at step $t-K$: $\bm\mu_\theta(\bm{x}_{t-K}, t-K) = \frac{1}{\sqrt{\alpha_{t-K}}}(\bm{x}_{t-K} - \frac{\beta_{t-K}}{\sqrt{1 - \Bar\alpha_{t-K}}}\bm\epsilon_\theta(\bm{x}_{t-K}, t-K))$ can be viewed as first computing an estimate of $x_0$ from the current-step error $\bm{\hat\epsilon}_\theta(\bm{x}_{t-K}, t-K)$:
    \vspace{-0.5mm}
    \begin{align*}
        \bm{\hat x}_0 = \frac{1}{\sqrt{\Bar\alpha_{t-K}}} (\bm{x}_{t-K} - \sqrt{1 - \Bar\alpha_{t-K}}\bm{\hat\epsilon}_\theta(\bm{x}_{t-K}, t-K))
    \end{align*}
    \vspace{-0.5mm}
    And use the estimate to compute the transition probability $q(\bm{x}_{t-K}|\bm{x}_{t-K}, \bm{x}_0)$.

\label{theory:close-form}
\end{proposition}

\begin{proof}
To avoid clustered notations, we use $t$ instead of $t - K$ for the proof below.
Starting from reorganizing \eqref{eq:forward} to the one step estimation:
\vspace{-0.5mm}
\begin{align}
    \bm{\hat x}_0 = \frac{1}{\sqrt{\Bar\alpha_t}} (\bm{x}_t - \sqrt{1 - \Bar\alpha_t}\bm{\hat\epsilon}_\theta(\bm{x}_t, t))
\end{align}
\vspace{-0.5mm}
where $\bm{\hat\epsilon}_\theta$ is the predicted error at step $t$ by the network.
Intuitively this equation means to use the current predicted error to one-step estimate $x_0$.
Using the Bayesian Theorem, one can show that
\vspace{-0.5mm}
\begin{align}
    &q(\bm{x}_{t-K}|\bm{x}_t, \hat{\bm{x}}_0) = \mathcal{N}(\bm{x}_{t-1};\Tilde{\bm\mu}(\bm{x}_t, \bm{x}_0), \Tilde\beta_t\bm{I})\\
     &\Tilde{\bm\mu}(\bm{x}_t, \bm{x}_0) = \frac{\sqrt{\Bar\alpha_{t-1}}\beta_t}{1 - \Bar\alpha_t}\bm{x}_0 + \frac{\sqrt{\alpha_t}(1 - \Bar\alpha_{t-1})}{1 - \Bar\alpha_t}\bm{x}_t
\end{align}
\vspace{-0.5mm}
If we plug $\hat{\bm{x}}_0$ into the above equation, it becomes:
\vspace{-0.5mm}
\begin{align}
    \bm\mu_\theta(\bm{x}_t, t) = \frac{1}{\sqrt{\alpha_t}}(\bm{x}_t - \frac{\beta_t}{\sqrt{1 - \Bar\alpha_t}}\bm\epsilon_\theta(\bm{x}_t, t))
\end{align}
\vspace{-0.5mm}
which is identical to the original modeling of DDPM~\cite{ddpm}.
\end{proof}

\section{The complete DPO-Diff algorithm}
\label{app:algorithm}

\begin{algorithm}
\caption{\textbf{DPO-Diff solver}: Discrete Prompt Optimization Algorithm}
\label{algo:dpo}
\begin{algorithmic}
\REQUIRE User Input $s_{user}$, diffusion model $G(\cdot)$, a loss function $\mathcal{L}(I, s)$, learning rate $lr$.
\ENSURE An optimized prompt $s^*$.
\STATE {// \textit{Building Search Space}}
\STATE Query ChatGPT to generate a word-substitutes dictionary for $s_{user}$
\STATE Initialize Gumbel parameter $\alpha$ accordingly.
\STATE {// \textit{Gradient Prompt Optimization}}
\FOR{$i$ from $1$ to max\_iter}
    \STATE Sample $p(w; \alpha)$ for each word $w$ from Gumbel Softmax.
    \STATE Compute mixed embedding: $\Tilde{e}(\alpha) = \sum_{i=1}^{|\mathcal{V}|} p(w=i;\alpha) * e_i$
    \STATE Compute text gradient: $g_{s} = \nabla_{\alpha} \mathcal{L}(G(\Tilde{e}(\alpha)), s)$
    \STATE Update Gumbel Parameter: $\alpha_i = \alpha_i - lr * g_{s_{user}}$
\ENDFOR
\STATE {// \textit{Evolutionary Sampling}}
\STATE Generate initial population $\mathcal{P} \sim Gumbel(\alpha)$
\STATE Find the population that minimizes $\mathcal{L}$ using genetic algorithm $\mathcal{P}^* = EvoSearch(\mathcal{P}, \mathcal{L})$
\STATE $s^* = \mathrm{argmax}_s(\mathcal{G}(s\in \mathcal{P}^*), s_{user})$
\end{algorithmic}
\end{algorithm}

\section{Taxonomy of prompt optimization v.s. textual inversion}
\begin{table}[h!]
    \centering
        \resizebox{0.98\linewidth}{!}{
        \begin{tabular}{lp{3cm}p{7cm}p{2.5cm}p{3cm}p{3cm}}
        \toprule
        \textbf{Task Name} & \textbf{Example Method} & \textbf{Taxonomy} & \textbf{Input} & \textbf{Output} & \textbf{Backpropagation} \\
        \midrule
        \midrule
        
        Textual Inversion & TI~\citep{ti}, NTI~\citep{nti}, PEZ~\citep{pez} & Generate novel visual concepts provided in user images, done by distilling image to a soft text embedding and use that for downstream tasks & use r image & a text prompt that encodes the given image content & identical to regular diffusion model training \\
        \hline
        Prompt Optimization & Promptist~\citep{promptist}, DPO-Diff (ours) & Improve the user prompt into a better one so that the generated images better follow the original user intention & user text prompt & An improved version of user text prompt & through inference steps \\

        \bottomrule
        \end{tabular}
    }
    \caption{Comparison of prompt optimization and textual inversion tasks.}
    \label{tab:taxonomy}
\end{table}

\section{Implementation details}
\label{app:impl}

\subsection{Hyperparameters}
\label{app:impl.hyperparam}
This section details the hyperparameter choices for our experiments.
We use the same set of hyperparameters for all datasets and tasks (prompt improvement and adversarial attack), unless otherwise specified.

\paragraph{Model}
We use Stable Diffusion v1-4 with a DDIM sampler for all experiments in the main paper.
The guidance scale and inference steps are set to 7.5 and 50 respectively (default).
We also experimented with other versions, such as Stable Diffusion v2-1 (512 x 512 resolution) and v2 (786x786 resolution), and found that the results are similar across different versions.
Although, we note that the high-resolution version of v2 tends to produce moderately better original images than v1-4 and v2-1 in terms of clip loss, possibly due to sharper images.

\paragraph{Shortcut Text Gradient}
We set $K = 1$, corresponding to a 1-step Shortcut Text Gradient.
This minimizes the memory and runtime cost while empirically producing enough signal to guide the prompt optimization.
Throughout the entire optimization episode, we progressively increase $t$ from 15 to 25 via a fixed stepwise function.
This corresponds to a coarse-to-fine learning curriculum.
We note that the performance is only marginally affected by the choice of the upper and lower bound for $t$ (e.g. 20-30, 10-40 all produce similar results), as long as it avoids values near $0$ (diminishing gradient) and $T$ (excessively noisy).

\paragraph{Gumbel softmax}
We use Gumbel Softmax with temperature 1.
The learnable parameters are initialized to 1 for the original word (for positive prompts) and empty string (for negative prompts), and 0 otherwise.
To encourage exploration. We bound the learnable parameters within 0 and 3 via hard clipping.
The performance remains largely incentive to the choice of bound, as long as they are in a reasonable range (i.e. not excessively small or large).

\paragraph{Optimization}
We optimize DPO-Diff using RMSprop with a learning rate of 0.1 and momentum of 0.5 for 20 iterations.
Each iteration will produce a single Gumbel Sample (batch size = 1) to compute the gradient, which will be clipped to $1/40$.

\paragraph{clip loss}
The specific clip loss used in our experiment is spherical clip loss, following an early online implementation of clip-guided diffusion~\citep{spherical}:
\begin{align*}
    \text{spherical\_clip(x, y)} = 2 \cdot \left(\arcsin\frac{\left\|x - y\right\|_2}{2}\right)^2
\end{align*}
Note that our method does not rely on this specific choice to function;
We also experimented with other distance measures such as cos similarity on the clip embedding space, and found that they produced nearly identical prompts (and thus images).

\paragraph{Evolution Search}
We follow a traditional evolution search composed of four steps: initialize population, tournament, mutation, and crossover.
The specific choice of hyperparameters is population size = 20, tournament = top 10, mutation with prob = 0.1 and size = 10, and crossover with size = 10.
We run the evolutionary search for two iterations for both tasks, while we note that the prompt improvement task often covers much faster (within a single iteration).

\subsection{Search space construction}
\label{app:impl.space}

We construct our Synonyms and Antonyms space by querying ChatGPT using the following prompts.
Since ChatGPT sometimes makes mistakes by producing false synonyms or antonyms, we further filter candidate prompts by thresholding the cosine similarity between adversarial prompts and user prompts in the embedding space of T5 during the evolutionary search phase~\cite{t5}. The threshold is set to 0.9 for all datasets.

\begin{center}
    \begin{quote}
        \texttt{Read the next paragraph. For each word, give 5 substitution words that do not change the meaning. Use the format of "A $\rightarrow$ B".} \\
    \end{quote}
\end{center}

For Antonyms:
\begin{center}
    \begin{quote}
        \texttt{Read the next paragraph. For each word, give 5 opposite words if it has any. Use the format of "A $\rightarrow$ B".}
    \end{quote}
\end{center}

\section{More experimental settings}
\label{app:exp_setting}

\subsection{Dataset Collection}
\label{app:exp_setting.dataset}

The prompts used in our paper are collected from three sources, DiffusionDB, COCO, and ChatGPT.

\paragraph{DiffusionDB}
DiffusionDB is a giant prompt database comprised of 2m highly diverse prompts for text-to-image generation.
Since these prompts are web-crawled, they are highly noisy, often containing incomplete phrases, emojis, random characters, non-imagery prompts, etc (We refer the reader to its \href{https://huggingface.co/datasets/poloclub/diffusiondb}{HuggingFace} repo for an overview of the entire database.).
Therefore, we filter prompts from DiffusionDB by (1). asking ChatGPT to determine whether the prompt is complete and describes an image, and (2) remove emoji-only prompts.
We filter a total of 4,000 prompts from DiffusionDB and use those prompts to generate images via Stable Diffusion.
We sample 100 prompts with clip loss above 0.85 for prompt improvement, and 0.8 for adversarial attacks respectively.
 For ChatGPT, we found that it tends to produce prompts with much lower clip score compared with COCO and DiffusionDB.
 To ensure a sufficient amount of prompts from this source is included in the dataset, we lower the cutoff threshold to 0.82 when filtering its hard prompts for the prompt improvement task.
 
\paragraph{COCO}
We use the captions from the 2014 validation split of MS-COCO dataset as prompts.
Similar to DiffusionDB, we filter 4000 prompts, and further sample 100 prompts with clip loss above 0.85 for prompt improvement, and 0.8 for adversarial attack respectively.

\paragraph{ChatGPT}
We also query ChatGPT for descriptions, as we found that it tends to produce more vivid and poetic descriptions compared with the former sources.
We use a diverse set of instructions for this task.
Below are a few example prompts we used to query ChatGPT for image descriptions.

\begin{center}
    \begin{quote}
        \texttt{Generate N diverse sentences describing photoes/pictures/images} \\
        \texttt{Generate N diverse sentences describing images with length around 10} \\
        \texttt{Generate N diverse sentences describing images with length around 20} \\
        \texttt{Generate N diverse sentences describing images using simple words} \\
        \texttt{Generate N diverse sentences describing images using fancy words}
    \end{quote}
\end{center}

Below are some example prompts returned by ChatGPT:

\begin{center}
    \begin{quote}
        \texttt{A majestic waterfall cascades down a rocky cliff into a clear pool below, surrounded by lush greenery.} \\
        \texttt{The sun setting behind the mountains casting a warm orange glow over the tranquil lake.} \\
        \texttt{A pair of bright red, shiny high heels sit on a glossy wooden floor, with a glittering disco ball above.} \\
        \texttt{A farmer plowing a field with a tractor.} \\
        \texttt{The vivid orange and dark monarch butterfly was flapping through the atmosphere, alighting on a flower to sip nectar.}
    \end{quote}
\end{center}

We empirically observe that ChatGPT produces prompts with low clip loss when used to generate images through Stable Diffusion on average, compared with DiffusionDB and COCO.
Therefore, for filtering challenging prompts, we reduce the threshold from 0.85 to 0.82 to allow more prompts to be selected.

\subsection{Human Evaluation}
\label{app:exp_setting.evaluation}
We ask 5 judges without ML background to evaluate the faithfulness of the generated images.
For each prompt, we generate two images using the same seeds across different methods.
To further avoid subjectiveness in evaluation, we provide the judgers an ordered list of important key concepts for each prompt, and ask them to find the winning prompt by comparing the hit rate.
The ordered list of key concepts is provided by ChatGPT.

Since the 600 prompts used in the main experiments are filtered automatically via clip loss, they exhibit a certain level of false positive rate: some images are actually faithful.
Therefore, we further filter out 100 most broken prompts to be evaluated by human judgers.

\paragraph{Special treatment for Adversarial Attack task.}
When conducting human evaluation on adversarial attack tasks, we make the following adjustments to the protocol:
(1). The wins and losses are reversed
(2) There will be no "draw", as this counts as a failed attempt.
(3). Removing meaning-altering successes: we asked the human evaluators to identify cases where success is achieved only because the adversarial prompt changed the meaning of the user prompt. Such instances are categorized as failures.
The results of our evaluation showcase that DPO-Diff achieved a success rate of 44\%, thereby establishing itself as the only baseline for this particular task on diffusion models.

\section{Extra qualitative results}
\label{app:exp_result}

We include extra quantitative results of DPO-Diff in Figure~\ref{fig:improve_more} and Figure~\ref{fig:attack_more}.
Additionally, we conducted experiments with the latest SD-XL model, as illustrated in ~\Cref{fig:xl.improve}.
The results indicate that DPO-Diff also achieves significant improvements with more advanced diffusion models.

    \begin{figure}
    \centering
    \caption{ More images generated by user input versus improved negative prompts using \textbf{Stable Diffusion v1-4}.}
        \begin{tabularx}{\textwidth}{>{\centering\arraybackslash}X|>{\centering\arraybackslash}X|>{\centering\arraybackslash}X}
            \hline\\[-7pt] {\textbf{User Input}} & {\textbf{Promptist - Modifiers}} & {\textbf{DPO-Diff - Negative Prompt}} \\[2pt]\hline



            \\[-7pt]
            {\prompt{The ash and dark pigeon was roosting on the lamppost, observing the environment.}} &
            {\prompt{intricate, elegant, highly detailed, ..., illustration, by justin gerard and artgerm, 8 k}} &
            {\prompt{fresh, shiny, hawk, overlooking, inside, Portrait, background, faded, unreal}} \\[1pt]
            \includegraphics[width=0.8\linewidth]{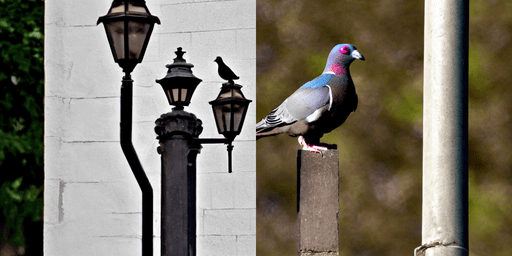} &
            \includegraphics[width=0.8\linewidth]{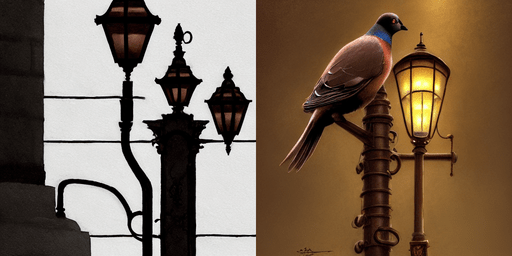} &
            \includegraphics[width=0.8\linewidth]{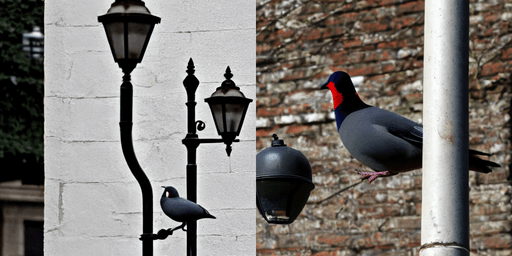}
            \\[2pt]\hline

            \\[-7pt]
            {\prompt{alien caught smoking cigarettes in rented house}} &
            {\prompt{intricate, elegant, highly detailed, ..., art by artgerm and greg rutkowski and, 8 k}} &
            {\prompt{native, liberated, clear, dull, out, bought, road, Macro, Script, monochrome, rendered}} \\[1pt]
            \includegraphics[width=0.8\linewidth]{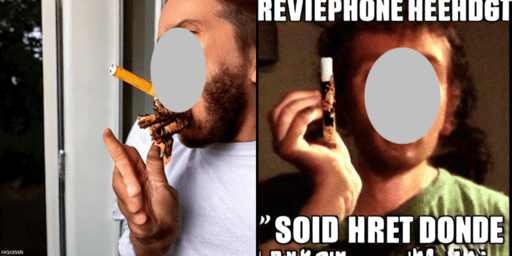} &
            \includegraphics[width=0.8\linewidth]{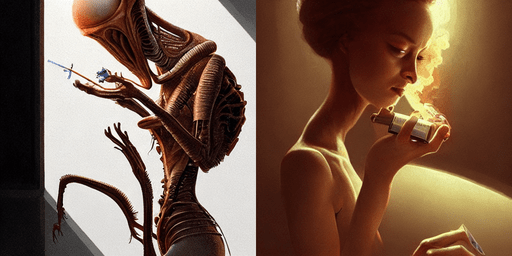} &
            \includegraphics[width=0.8\linewidth]{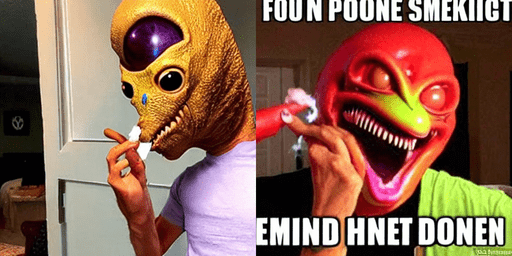}
            \\[2pt]\hline

            \\[-7pt]
            {\prompt{a spooky ghost in a graveyard by justin gerard and tony sart}} &
            {\prompt{greg rutkowski, zabrocki, karlkka, ..., zenith view, zenith view, pincushion lens effect}} &
            {\prompt{physical, house, aside, except, Grains, design, replica}} \\[1pt]
            \includegraphics[width=0.8\linewidth]{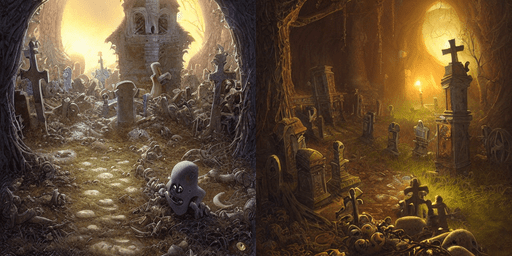} &
            \includegraphics[width=0.8\linewidth]{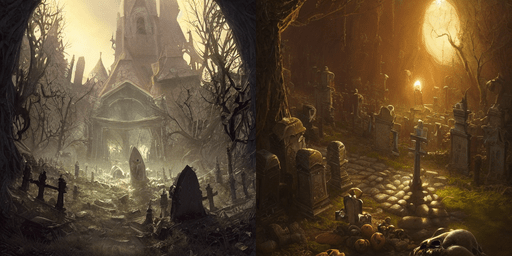} &
            \includegraphics[width=0.8\linewidth]{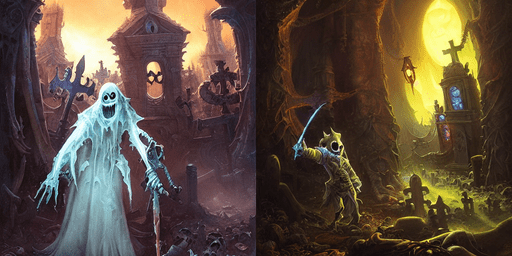}
            \\[2pt]\hline



            \\[-7pt]
            {\prompt{a plane flies through the air with fumes coming out the back }} &
            {\prompt{Rephrase: a plane flies through the air with fumes coming ..., trending on artstation}} &
            {\prompt{car, crashes, land, ..., breeze, departing, into, front, Grains, cold, monochrome, oversized}} \\[1pt]
            \includegraphics[width=0.8\linewidth]{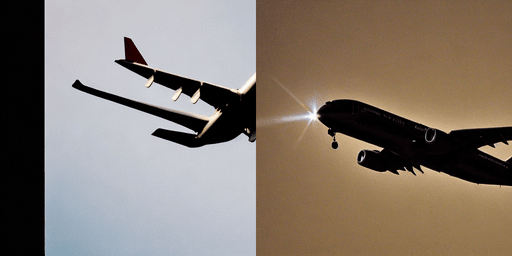} &
            \includegraphics[width=0.8\linewidth]{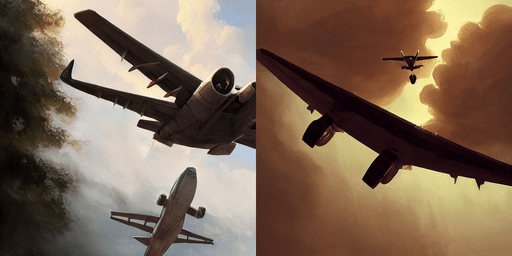} &
            \includegraphics[width=0.8\linewidth]{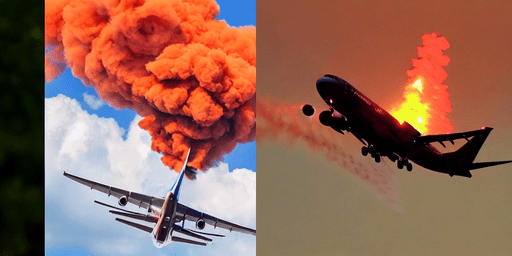}
            \\[2pt]\hline



            \\[-7pt]
            {\prompt{A man is seated on a floor with a computer and some papers.}} &
            {\prompt{intricate, elegant, highly detailed, ..., illustration, by justin gerard and artger rutkowski, 8 k}} &
            {\prompt{female, was, standing, below, top, without, zero, ..., emails, Blurry, bad, extra, proportion}} \\[1pt]
            \includegraphics[width=0.8\linewidth]{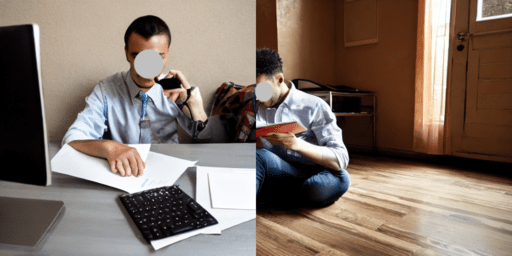} &
            \includegraphics[width=0.8\linewidth]{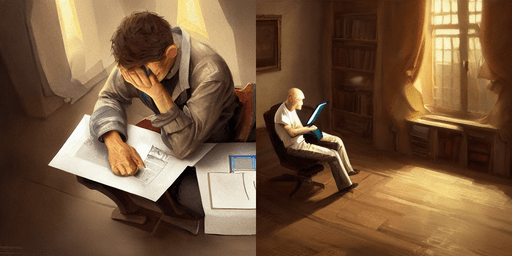} &
            \includegraphics[width=0.8\linewidth]{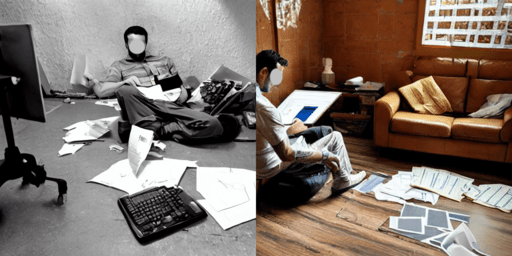}
            \\[2pt]\hline


            \\[-7pt]
            {\prompt{Orange and brown cat sitting on top of white shoes.}} &
            {\prompt{Trending on Artstation, ..., 4k, 8k, unreal 5, very detailed, hyper control-realism.}} &
            {\prompt{purple, however, black, crawling, ..., socks, Cropped, background, inverted, shape}} \\[1pt]
            \includegraphics[width=0.8\linewidth]{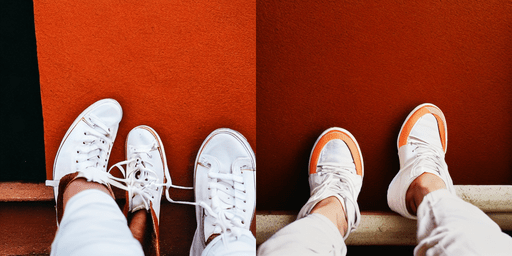} &
            \includegraphics[width=0.8\linewidth]{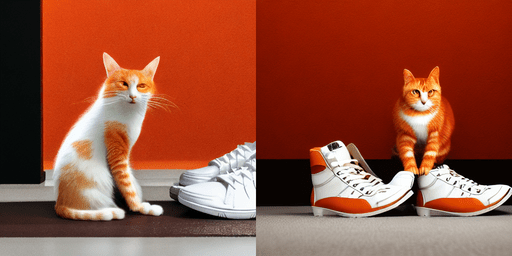} &
            \includegraphics[width=0.8\linewidth]{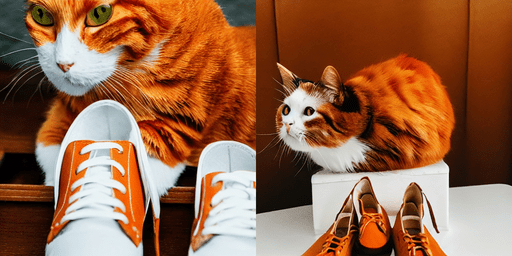}
            \\[2pt]\hline

        \end{tabularx}
    \label{fig:improve_more}
    \end{figure}

    \begin{figure}
    \centering
    \caption{More images generated by user input and adversarial prompts using \textbf{Stable Diffusion v1-4}.}

            \begin{tabularx}{\textwidth}{>{\centering\arraybackslash}X|>{\centering\arraybackslash}X}
                \hline\\[-10pt] {\textbf{User Input}} & {\textbf{DPO-Diff - Adversarial Prompts}} \\[0pt]\hline

                \\[-7pt]
                {\prompt{A cinematic scene from Berlin.}} &
                {\prompt{A cinematic shot from Metropolis.}} \\[2pt]
                \includegraphics[width=0.8\linewidth]{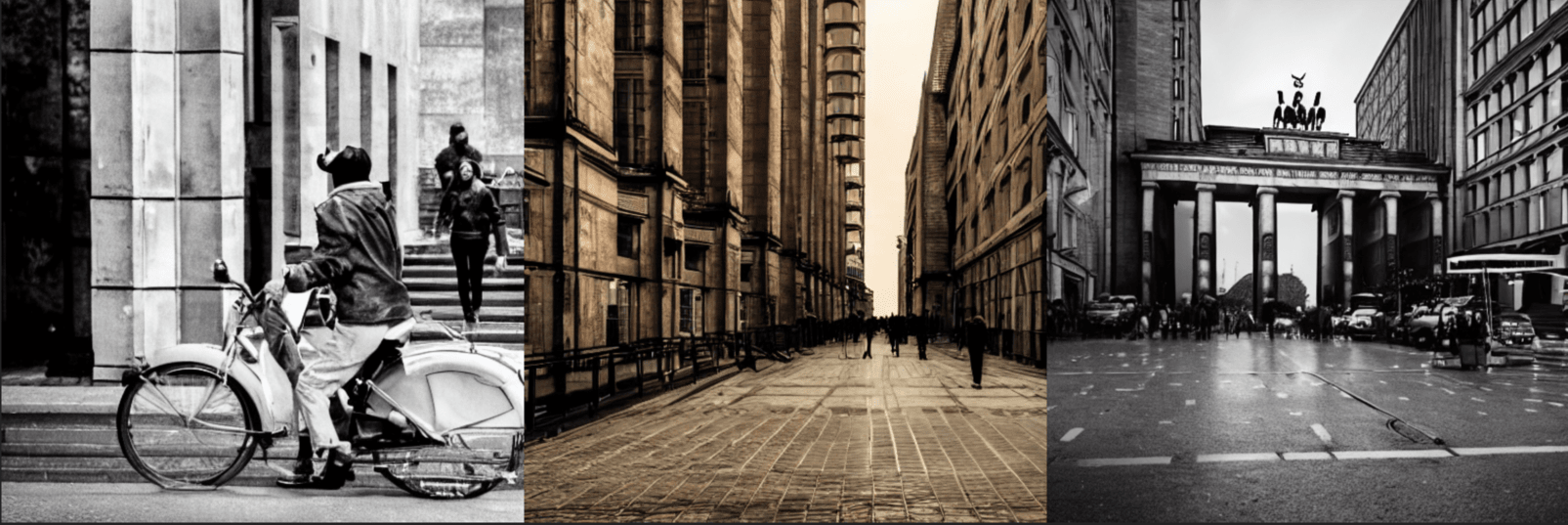} &
                \includegraphics[width=0.8\linewidth]{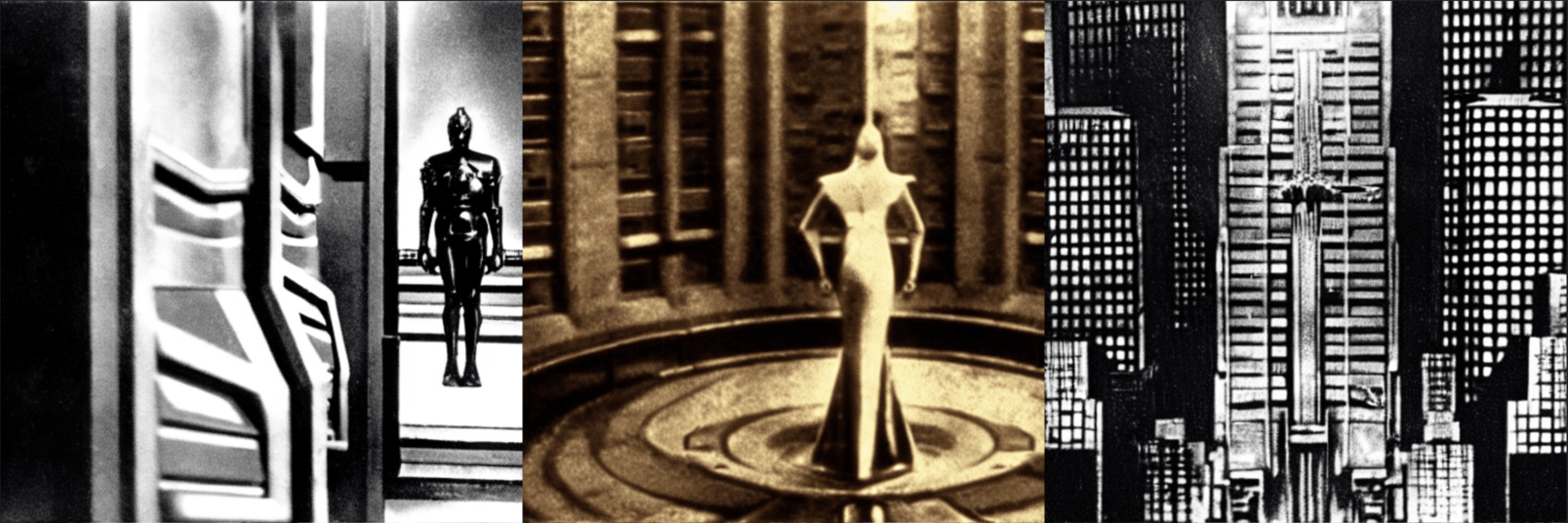}
                \\[2pt]\hline
    
    


                \\[-12pt]
                {\prompt{A painter adding the finishing touches to a vibrant canvas.}} &
                {\prompt{A craftsman incorporating the finishing touches to a vivid masterpiece .}} \\[1pt]
                \includegraphics[width=0.8\linewidth]{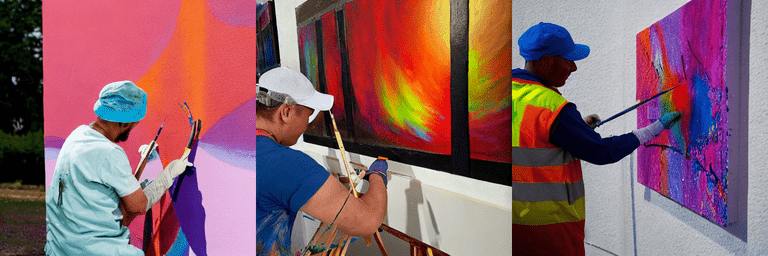} &
                \includegraphics[width=0.8\linewidth]{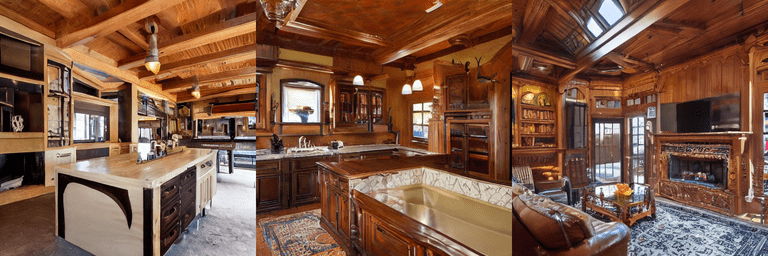}
                \\[-2pt]\hline

                \\[-12pt]
                {\prompt{A skillful tailor sewing a beautiful dress with intricate details.}} &
                {\prompt{A skillful tailor tailoring a lovely attire with sophisticated elements .}} \\[1pt]
                \includegraphics[width=0.8\linewidth]{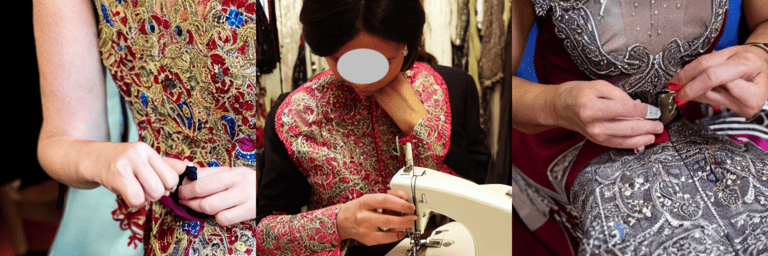} &
                \includegraphics[width=0.8\linewidth]{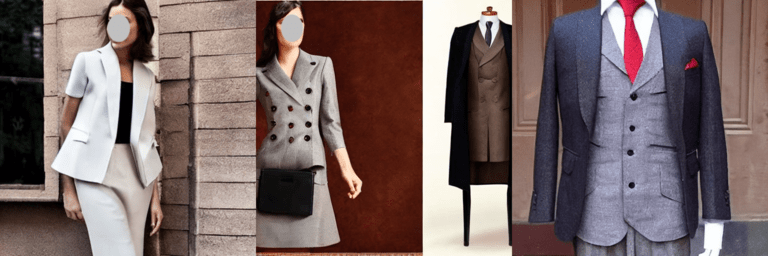}
                \\[-2pt]\hline

                \\[-12pt]
                {\prompt{portrait of evil witch woman in front of sinister deep dark forest ambience}} &
                {\prompt{image of vile mage dame in front of threatening profound dim wilderness ambience}} \\[1pt]
                \includegraphics[width=0.8\linewidth]{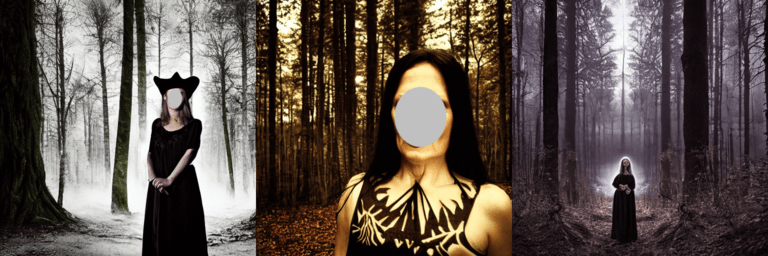} &
                \includegraphics[width=0.8\linewidth]{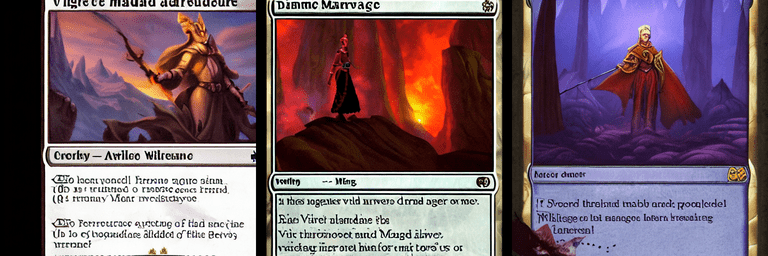}
                \\[-2pt]\hline



                \\[-12pt]
                {\prompt{Amazing photorealistic digital concept art of a guardian robot in a rural setting by a barn.}} &
                {\prompt{astounding photorealistic digital theory design of a defender robot in a provincial context by a stable .}} \\[1pt]
                \includegraphics[width=0.8\linewidth]{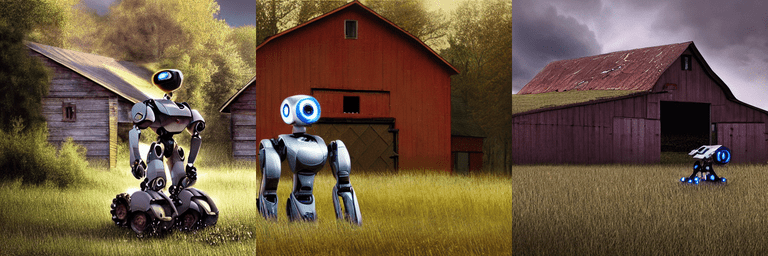} &
                \includegraphics[width=0.8\linewidth]{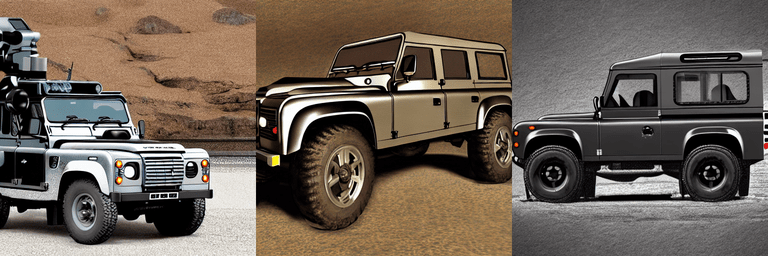}
                \\[-2pt]\hline

                \\[-12pt]
                {\prompt{close up portrait of a young lizard as a wizard with an epic idea}} &
                {\prompt{close up snapshot of a youthful chameleon as a magician with an heroic guess}} \\[1pt]
                \includegraphics[width=0.8\linewidth]{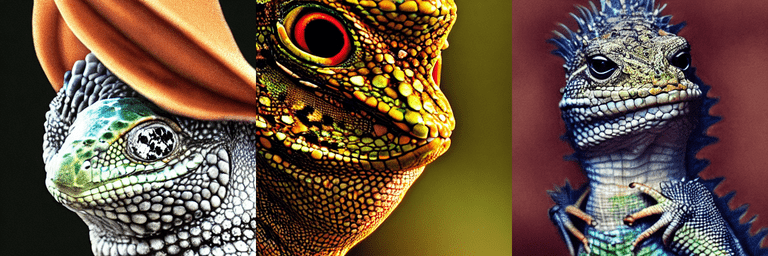} &
                \includegraphics[width=0.8\linewidth]{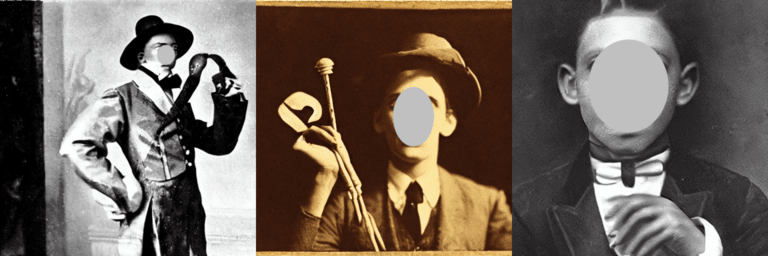}
                \\[-2pt]\hline
        \end{tabularx}
    \label{fig:attack_more}
    \end{figure}

\newpage

\section{Further discussion on Gradient-based Prompt Optimization}

The computational cost of the Shortcut Text Gradient is controlled by $K$. Moreover, when we set $t = T$ and $K = T - 1$, it becomes the full-text gradient.

The result of remark 2 is rather straightforward:
recall that the image generation process starts with a random noise $x_T$ and gradually denoising it to the final image $x_0$.
Since the gradient is enabled from $t$ to $t - K$ in Shortcut Text Gradient; when $t = T$ and $K = T$, it indicates that gradient is enabled from $T$ to $0$, which covers the entire inference process.
In this case, the Shortcut Text Gradient reduces to the full gradient on text.


\begin{figure*}[t!]
    \centering

    \includegraphics[width=0.45\linewidth]{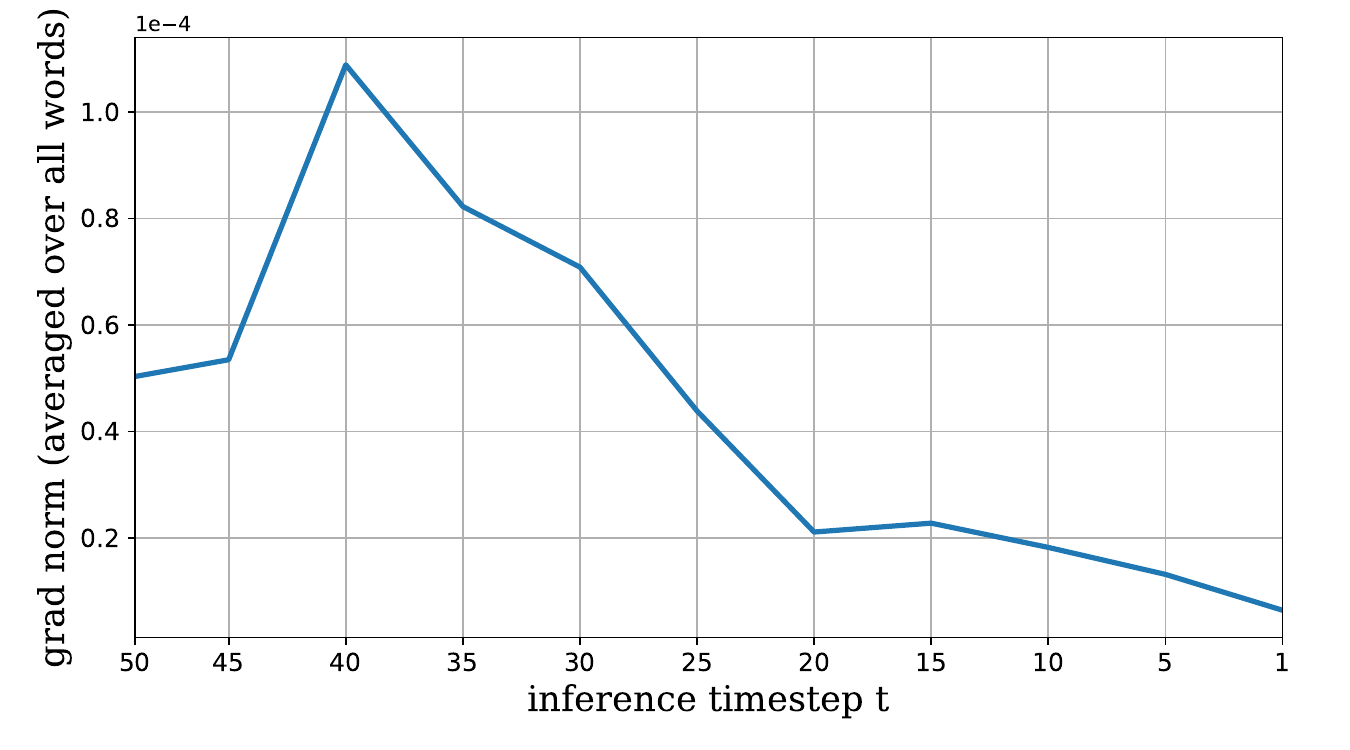}
    \label{fig:ablate_t.attack}
    \includegraphics[width=0.45\linewidth]{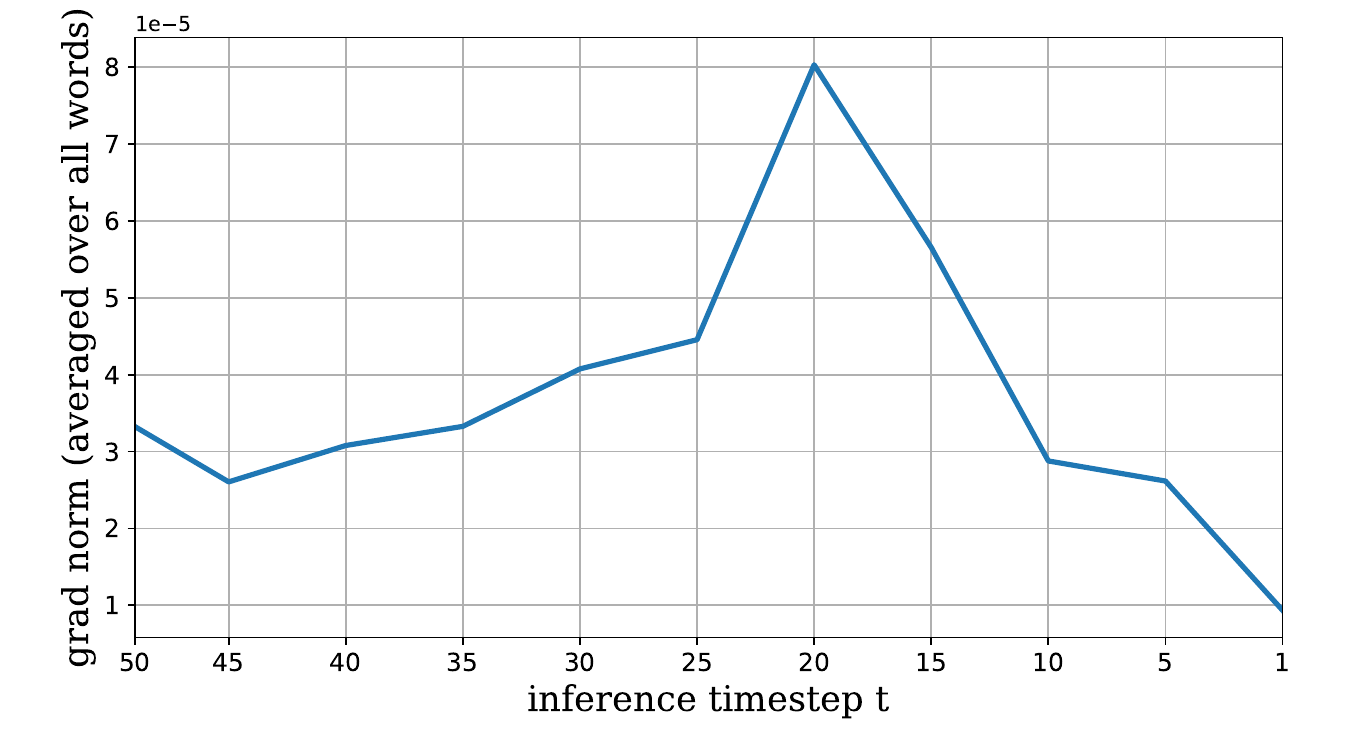}
    \label{fig:ablate_t.improve}

    \caption{\textbf{Gradient near the beginning and end of the inference process are significantly less informative}. We plot the average gradient norm over all words across different timesteps. For each timestep, the Shortcut Text Gradient is computed over 100 Gumbel samples.}
    \label{fig:ablate_t}
    
\end{figure*}

\section{Extra ablation study results.}
\subsection{Gradient norm v.s. timestep.}
When randomly sampling t in computing the Shortcut Text Gradient, we avoid timesteps near the beginning and the end of the image generation process, as gradients at those places are not informative.
As we can see, for both adversarial attack and prompt improvement, the gradient norm is substantially smaller near $t=T$ and especially $t=0$, compared with timesteps in the middle.
The reason, we conjecture, is that the images are almost pure noise at the beginning, and are almost finalized towards the end.
Figure \ref{fig:ablate_t} shows the empirical gradient norm across different timesteps.

\subsection{Extended discussion on different search algorithms}
In our experiments, we found that Gradient-based Prompt Optimization converges faster at the early stage of the optimization.
This result confirms the common belief that white-box algorithms are more query efficient than black-box algorithms in several other machine learning fields, such as adversarial attack~\cite{blackbox, opt}.
However, when giving a sufficient amount of query, Evolutionary Search eventually catches up and even outperforms GPO.
The reason, we conjecture, is that GPO uses random search to draw candidates from the learned distribution, which bottlenecked its sample efficiency at later stages.
This promotes the hybrid algorithm used in our experiments: Using Evolutionary Search to sample from the learned distribution of GPO.
The hybrid algorithm achieves the best overall convergence.

\subsection{Extended discussion on negative v.s. positive prompt optimization}
As discussed in the main text, one of our highlighted findings of is that optimizing for negative prompts is more effective than positive prompts in improving the prompt-following ability of diffusion models.
This is evidenced by Table \ref{tab:ablate.nvp}, which shows that Antonym Space contains a denser population of promising prompts (lower clip loss) than positive spaces.
Such search space also allows the search algorithm to identify an improved prompt more easily.
We conjecture that this might indicate diffusion models are more sensitive to changes in negative prompts than positive prompts, as the baseline negative prompt is merely an empty string.

\begin{figure}
\centering
\caption{Images generated by user input and improved negative prompts on \textbf{Stable Diffusion XL}.}
    \begin{tabularx}{\textwidth}{>{\centering\arraybackslash}X|>{\centering\arraybackslash}X|>{\centering\arraybackslash}X}
        \hline\\[-7pt] {\textbf{User Input}} & {\textbf{Promptist - Modifiers}} & {\textbf{DPO-Diff - Negative Prompt}} \\[2pt]\hline

        \\[-7pt]
        {\prompt{a brown dachshund with a black cat sitting in a canoe.}} &
        {\prompt{highly detailed, digital painting, ..., sharp focus, illustration, art by artgerm and greg rutkowski and epao}} &
        {\prompt{zero, black, cat, lacking, green, horse, walking, beyond, house, Mutation, animals, error, surreal}} \\[1pt]
        \includegraphics[scale=0.5,width=0.8\linewidth]{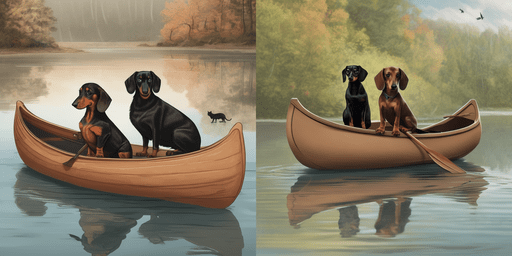} &
        \includegraphics[scale=0.5,width=0.8\linewidth]{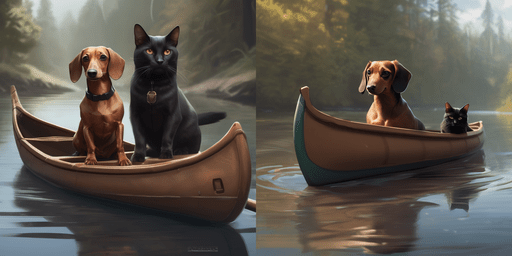} &
        \includegraphics[scale=0.5,width=0.8\linewidth]{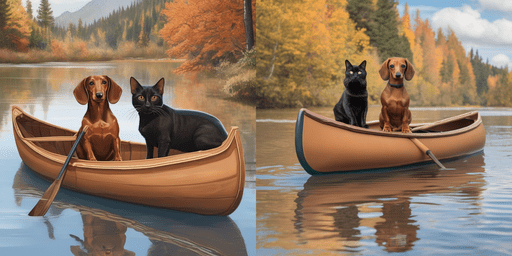}
        \\[2pt]\hline

        \\[-7pt]
        {\prompt{darth vader in iron man armour}} &
        {\prompt{highly detailed, digital painting, ..., illustration, art by greg rutkowski and alphonse mucha}} &
        {\prompt{yoda, outside, lightweight, exposed, Render, Script, incomplete, pieces}} \\[1pt]
        \includegraphics[scale=0.5,width=0.8\linewidth]{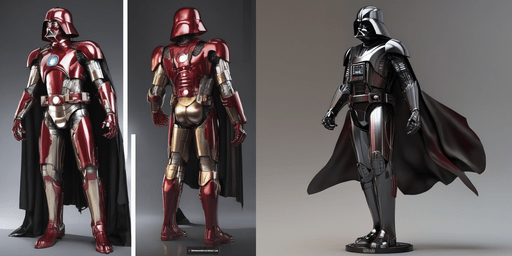} &
        \includegraphics[scale=0.5,width=0.8\linewidth]{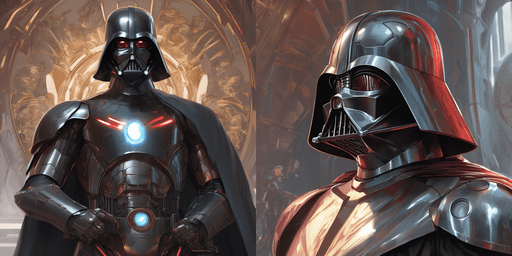} &
        \includegraphics[scale=0.5,width=0.8\linewidth]{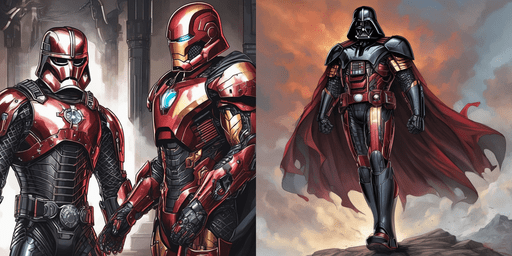}
        \\[2pt]\hline

        \\[-7pt]
        {\prompt{The ash and dark pigeon was roosting on the lamppost, observing the environment.}} &
        {\prompt{intricate, elegant, highly detailed, digital painting, artstation, concept art, sharp focus, illustration, by justin gerard and art rutkowski, 8 k}} &
        {\prompt{green, clear, departing, ditch, inner, Mistake, CGI, cooked, replica}} \\[1pt]
        \includegraphics[scale=0.1,width=0.8\linewidth]{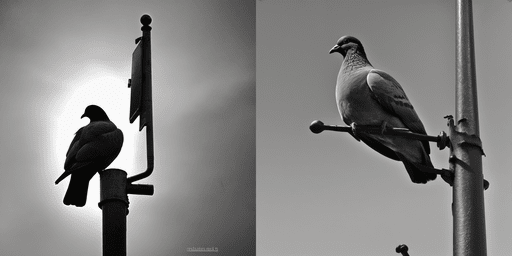} &
        \includegraphics[scale=0.1,width=0.8\linewidth]{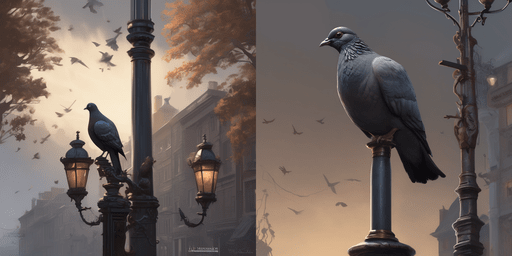} &
        \includegraphics[scale=0.1,width=0.8\linewidth]{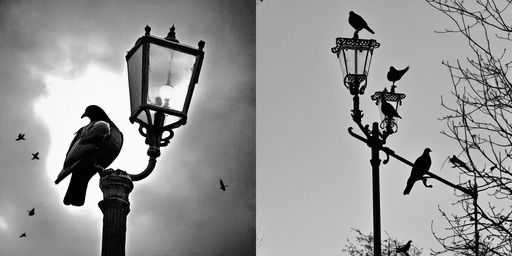}
        \\[2pt]\hline

        \\[-7pt]
        {\prompt{a very big building with a mounted clock}} &
        {\prompt{greg rutkowski, zabrocki, ..., 8 k, ultra wide angle, zenith view, pincushion lens effect}} &
        {\prompt{mildly, tiny, detached, Logo, cityscape, inverted, stale}} \\[1pt]
        \includegraphics[scale=0.1,width=0.8\linewidth]{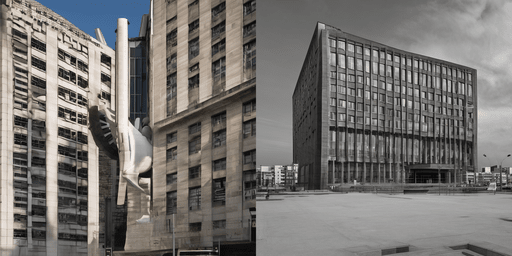} &
        \includegraphics[scale=0.1,width=0.8\linewidth]{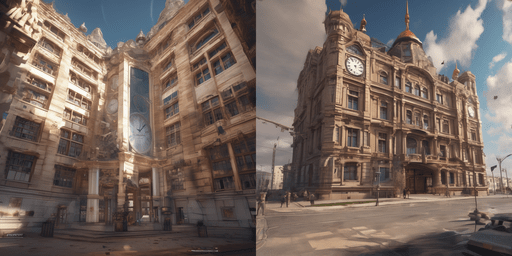} &
        \includegraphics[scale=0.1,width=0.8\linewidth]{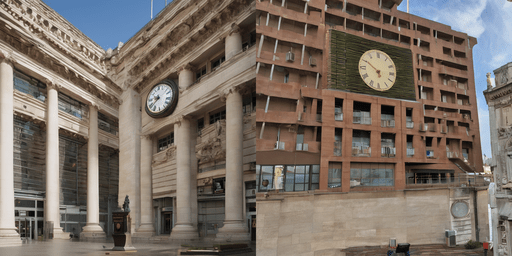}
        \\[2pt]\hline

        \\[-7pt]
        {\prompt{The man is sitting on the bench close to the asian section.}} &
        {\prompt{greg rutkowski, zabrocki, karlkka, ..., 8 k, ultra wide angle, zenith view, pincushion lens effect}} &
        {\prompt{girl, standing, under, ground, distant, unto, entirety, Mistake, black, engine, poorly}} \\[1pt]
        \includegraphics[scale=0.1,width=0.8\linewidth]{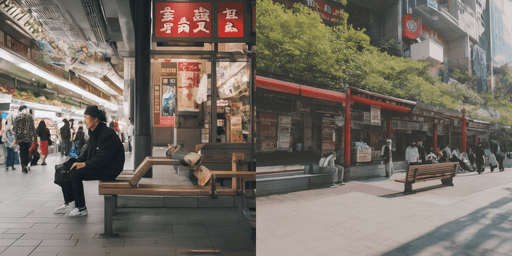} &
        \includegraphics[scale=0.1,width=0.8\linewidth]{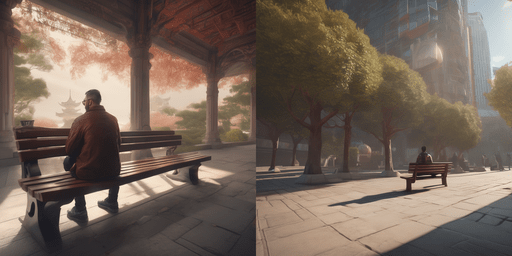} &
        \includegraphics[scale=0.1,width=0.8\linewidth]{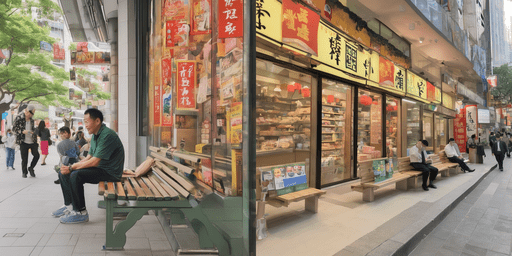}
        \\[2pt]\hline

        \\[-7pt]
        {\prompt{Two sinks stand next to a bathtub in a bathroom.}} &
        {\prompt{greg rutkowski, zabrocki, karlkka, jayison devadas, trending impervious}} &
        {\prompt{one,soars, lie, multiple, kitchen, outside, bedroom, Blurry, artificial, down, poorly}} \\[1pt]
        \includegraphics[scale=0.1,width=0.8\linewidth]{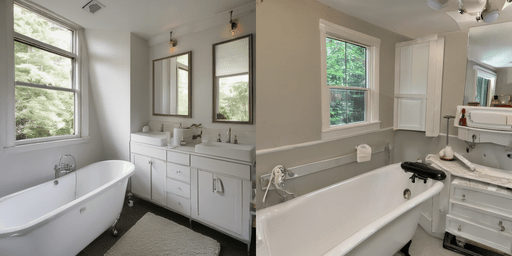} &
        \includegraphics[scale=0.1,width=0.8\linewidth]{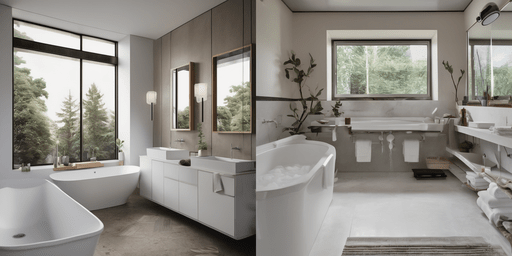} &
        \includegraphics[scale=0.1,width=0.8\linewidth]{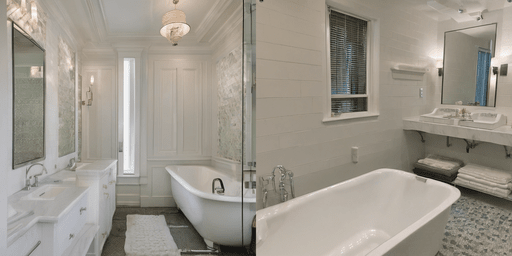}
        \\[2pt]\hline


        \\[-7pt]
        {\prompt{A woman that is standing next to a man.}} &
        {\prompt{highly detailed, digital painting, artstation, ..., art by greg rutkowski and alphonse mucha}} &
        {\prompt{male, crawling, away, far, several, woman, Mutation, characters, folded, username}} \\[1pt]
        \includegraphics[scale=0.1,width=0.8\linewidth]{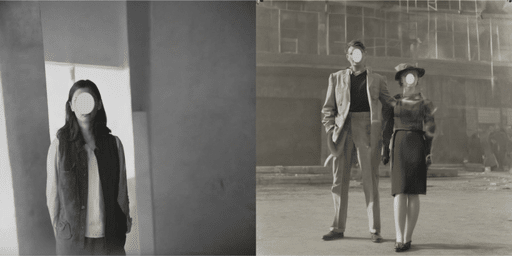} &
        \includegraphics[scale=0.1,width=0.8\linewidth]{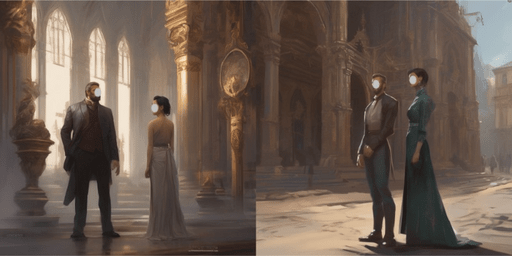} &
        \includegraphics[scale=0.1,width=0.8\linewidth]{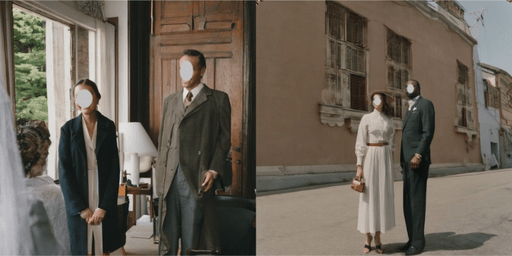}
        \\[2pt]\hline

    \end{tabularx}
\label{fig:xl.improve}
\end{figure}

\end{document}